\documentclass{article}
\usepackage{PRIMEarxiv}
\usepackage[utf8]{inputenc}
\usepackage[T1]{fontenc}
\usepackage{amsmath}
\usepackage{amssymb}
\usepackage{mathtools}
\usepackage{amsthm}
\usepackage{graphicx}
\usepackage{xcolor}
\usepackage{booktabs}
\usepackage{microtype}
\usepackage{algorithm}
\usepackage{algorithmic}
\usepackage[numbers,sort&compress]{natbib}
\usepackage{hyperref}
\usepackage[capitalize,noabbrev]{cleveref}

\graphicspath{{figs/}}
\providecommand{\ket}[1]{\left|#1\right\rangle}
\providecommand{\bra}[1]{\left\langle#1\right|}
\providecommand{\braket}[2]{\left\langle#1\,\middle|\,#2\right\rangle}

\theoremstyle{plain}
\newtheorem{theorem}{Theorem}[section]
\newtheorem{proposition}[theorem]{Proposition}
\newtheorem{lemma}[theorem]{Lemma}

\theoremstyle{definition}

\theoremstyle{remark}
\newtheorem{remark}[theorem]{Remark}

\title{The Quantum Shortcut:\\ Complex Phase-State Dynamics Reduce the Optimization Steps of Sequence Models}

\author{
Ahmed Nebli\thanks{Corresponding author.}\\
cAI Technology GmbH\\
\texttt{ahmed.nebli@cai-technology.ai}
\And
Hadi Saadatdoorabi\\
cAI Technology GmbH\\
\texttt{hadi.saadatdoorabi@cai-technology.ai}
\And
Christopher Keibel\\
Independent Researcher\\
\texttt{christopher.keibel.90@gmail.com}
\And
Kevin Yam\\
Yam Technology Consulting\\
\texttt{kevin@yam-consulting.com}
}
\date{}

\begin{document}
\maketitle

\begin{abstract}
{\color{black}Sequence models are conventionally distinguished by their backbone, the mechanism that routes information across positions, such as attention or recurrence. This paper varies a choice that is prior to the backbone and shared by nearly all current models: the \emph{substrate}, the number system in which the hidden state is represented together with the form of the map from state to prediction. The prevailing substrate is a real-valued state with an affine--softmax readout; we study a complex-valued alternative drawn from the mathematics of quantum theory, in which information is carried by the phases of the state and scores are quadratic Born forms. Prior work proved an idealized version of this substrate representationally stronger than any real model with a linear readout; we ask whether it also trains faster. Relaxing the two properties that block deployment, exact unitarity and the Born vocabulary readout, we instantiate it in the Mamba state-space model and an attention-based Transformer. At 253M parameters, matched to within $0.02\%$ and trained under one fixed protocol on three byte-level corpora, the complex models reach every measured validation loss in approximately one third (state-space) and one half (attention) of the optimization steps of their real counterparts. The two backbones then diverge. Once the learning-rate warmup ends, the state-space advantage continues to widen, from $0.321$ to $0.354$ bits per character on OpenWebText and from $0.368$ to $0.396$ on FineWeb, which an artifact of the warmup ramp would not do; the attention advantage instead decays toward zero on every corpus, and is therefore an effect of early training. The comparison counts optimization steps at equal tokens per step, and for the attention pair the current kernel overhead cancels the step advantage in wall-clock time. We analyze the two candidate mechanisms and release all code and training logs.}
\end{abstract}

\color{black}

\section{Introduction}
\label{sec:intro}
Sequence modeling, the task of predicting the next element of a sequence from the elements that precede it, underlies modern language models and much of contemporary machine learning. The design of a sequence model is conventionally described by its backbone, the rule by which information from earlier positions reaches later ones. Transformers aggregate by global self-attention~\citep{vaswani2017attention}; state-space models propagate a fixed-dimensional linear recurrence~\citep{gu2021efficiently,gu2023mamba}; long-convolution models mix positions through learned kernels~\citep{poli2023hyena}. The trade-offs among these mechanisms, in expressivity, parallelism, and long-range retention, have been characterized in detail~\citep{tay2022efficient,tay2021long,gu2021efficiently}. A second design choice, made implicitly by nearly all of these models, has been varied only rarely (the exceptions are reviewed in \Cref{sec:related}): the number system in which the hidden state is represented, together with the functional form of the map from state to next-token distribution. In the prevailing design the state is a vector in $\mathbb{R}^{d}$ and the output map is an affine transformation followed by a softmax. We call this pair, the field of the state and the form of the readout, the model's \emph{substrate}, and we treat it as an axis of design distinct from the backbone; whether it is also independent of the backbone is precisely what the experiments test (\Cref{sec:universality}).

The substrate matters because of how forgetting works. A real-valued model forgets by shrinking the entries of its state, and shrinking a real number degrades the information it carries; attenuation of one percent per position, compounded over a thousand positions, leaves less than $0.005\%$ of the signal from the earliest positions. A complex number carries two separable quantities, a magnitude and a phase. The substrate studied here stores information in the phases, implements forgetting as a contraction of the magnitudes alone, and reads predictions from phase differences through a quadratic readout, under which two contributions cancel when their phases oppose; this is destructive interference, the same phenomenon exploited in acoustic noise cancellation. Contraction leaves phases unchanged, so stored information fades in amplitude without being distorted, and {\color{black}interference permits evidence to count against a hypothesis rather than merely attenuate it, with a strength set by the participating magnitudes.} The empirical finding is that models built on this substrate reach every quality level we measure in approximately one third (state-space backbone) to one half (attention backbone) of the optimization steps of real-valued models of identical size trained identically, on corpora spanning 100~MB to 15~GB. The substrate borrows only the mathematics of quantum theory; the models are classical and run on standard accelerators.

The real, affine--softmax substrate carries two documented limitations, and both are independent of the backbone. First, an affine map from a $d$-dimensional state followed by a softmax realizes log-probability matrices of rank at most $d+1$~\citep{yang2018breaking}. When the conditional structure of the data has higher rank, the readout is misspecified no matter how the state is computed. Second, a real linear recurrence with spectral radius $\rho<1$ attenuates a propagated gradient by $\rho^{T}$ over $T$ steps, the classical vanishing-gradient regime~\citep{bengio1994learning,pascanu2013difficulty}. At operating points typical of deep state-space models, say $\rho=0.99$ and $T=1024$, this factor is $\rho^{T}\approx 3.4\times 10^{-5}$: the gradient path to early positions is effectively severed. Neither limitation is removed by changing how the model routes information. The two limitations also indicate where a remedy must act: the readout must escape the rank ceiling, and the recurrence must carry information in a quantity that contraction does not attenuate. The complex substrate, which pairs a quadratic readout with information stored in phase, does both, and \Cref{sec:theory} states each point precisely.

Both limitations bear on training as well as on representation, and they motivate the quantity we measure: the number of optimization steps required to reach a target loss. This number determines the computation expended and, at a fixed batch size, the volume of data consumed in reaching a given quality; a substrate that reduces it improves the training cost and the sample efficiency of every backbone that adopts it. As the results will show, the difference between substrates is established in the earliest steps of training and is largest there. This per-step quantity, rather than the terminal loss alone, is the primary measure of this study.

The motivation for changing the field from $\mathbb{R}$ to $\mathbb{C}$ is a specific structural property: interference. In a real vector space the combination of two state components is monotone in their magnitudes. In a complex vector space it is additionally governed by their relative phase, so that two components of fixed magnitude can reinforce or cancel. A simple example, to which we return throughout the paper, illustrates the consequence for language modeling. A byte-level model that has read the characters \texttt{1 9 4} must maintain two readings of its context: the prefix of a year, as in \texttt{1945}, and the prefix of a price, as in \texttt{19.4}. A real model holds both readings as magnitudes and, once a later byte settles the question, suppresses the losing reading with a learned gate that shrinks its magnitude toward zero. A complex model can instead advance the relative phase of the two readings so that, under a quadratic readout, the losing reading interferes destructively with the incoming evidence while the winning one interferes constructively. The distinction is substantive because cancellation strictly extends attenuation: a gate can only drive a component toward zero, whereas interference allows a component that is still present to count against an outcome. The same enrichment is what renders alternating-current analysis, Fourier optics, and quantum mechanics tractable over the complex field.

The complex substrate was developed in idealized form by \citet{nebli2026quantum}. There the hidden state is a unit-norm vector in $\mathbb{C}^N$; its evolution is generated by a learned Hamiltonian and is exactly unitary, discretized by a Cayley (Crank--Nicolson) map that preserves the state norm for any step size; and the next-token distribution is the Born rule $p_v=|\braket{e_v}{\psi}|^2$, a quadratic form in the state. That work proved a representational separation: a family of phase-disambiguation tasks solved exactly by a complex unitary model of dimension $N$ requires dimension $\Omega(N^2)$ from any real orthogonal model with an affine--softmax readout, {\color{black}because the quadratic readout accesses the $\Theta(N^2)$ pairwise phase correlations of the state's density matrix, which a linear functional of the state cannot resolve; these correlations are $\Theta(N^2)$ measurement directions, not $\Theta(N^2)$ independent state coordinates, since the pure-state variety has real dimension $2N-1$.} The separation is representational: it says what each class can express at its optimum and is silent on how training gets there. Whether the substrate helps a \emph{trained} model, at scale, on natural data, was left open.

Two properties of the idealized substrate obstruct direct deployment, and this paper studies the consequences of relaxing them. Exact unitarity admits no forgetting: a norm-preserving recurrence cannot damp a state component that has ceased to matter, whereas gating has served as the forgetting mechanism of recurrent models since the LSTM~\citep{hochreiter1997long,gers2000learning}, and the ablations of \citet{gu2023mamba} attribute the gains of selective state-space models over their time-invariant predecessors to exactly this input-dependent forgetting. And a Born readout over a vocabulary of size $V$ costs $O(NV)$ per step and ties the state width to $\sqrt{V}$ if its rank advantage is to be realized, which is unfavorable at subword vocabularies. We therefore construct a \emph{deployable} substrate: unitarity is relaxed to a contractive, input-dependent recurrence that keeps the phase of every transition eigenvalue on the unit circle, and the Born rule is applied where it is inexpensive, as the scoring function of a linear-attention kernel, while the vocabulary readout remains a tied softmax. \Cref{sec:substrate} states precisely what each relaxation surrenders and what it retains.

The principal empirical result is that this deployable substrate, at 253M parameters and under a single training protocol, reaches every measured validation loss in approximately one third of the optimization steps of the real Mamba and one half of those of the real Transformer, on three byte-level corpora spanning a $150\times$ range of size. We refer to this reduction as the \emph{quantum shortcut}. {\color{black}It is established early in training: H-Mamba passes 2.5 bits per character near step 78, a level the real Mamba reaches near step 252, and the ratio of steps remains nearly constant as the target descends from 2.5 to 2.0 bits per character. Because the measured targets are crossed during the learning-rate warmup, we examine separately what happens once the warmup ends (\Cref{sec:postwarmup}). The two backbones behave differently there, and the difference is part of the result: the state-space gap continues to widen after the ramp stops, while the attention gap decays toward zero. The acceleration is thus backbone-independent in the early phase, and, on the present evidence, persistent only for the recurrence.} {\color{black}The design of the comparison addresses the most immediate alternative explanations. The members of each pair are matched in parameter count to within $0.02\%$, which rules out raw parameter count as the driver; the training protocol is identical across models and follows the published recipe at this scale~\citep{brown2020language}, which argues against a tuning artifact; the stability of the ratio across the descending range of the loss is inconsistent with a transient advantage of initialization; and the recurrence of the effect across two backbones that share almost no computational structure makes a purely backbone-specific explanation unlikely.} One qualification applies. The fused kernel of the complex attention currently runs at half the throughput of its real counterpart, so for that pair the reduction in steps yields parity, not savings, in wall-clock time, and the state-space comparison lacks an optimized real baseline kernel; the established value of the shortcut is accordingly in optimization steps and data efficiency rather than in wall-clock time (\Cref{sec:throughput}).

This paper makes three contributions. First, we define the deployable complex substrate and characterize its two relaxations, contraction in place of unitarity and Born scoring in place of a Born vocabulary head, stating for each what is surrendered and what is retained (\Cref{sec:substrate}). Second, in the central result of the paper, we establish the quantum shortcut: at matched capacity and under an identical protocol, the substrate reduces the number of steps to every measured target by a factor of approximately three for a state-space backbone and two for an attention backbone, uniformly over three corpora (\Cref{sec:setup,sec:results}; \Cref{fig:hero,fig:steps}; \Cref{tab:learning_speed,tab:absolute}). Third, we account for the effect from both directions: an optimization floor relates the rank of an affine--softmax readout to the loss it can attain, {\color{black}an exact-transport lemma shows that the phase coordinate crosses the contracting recurrence with unit sensitivity, while its observable effect remains scaled by the participating magnitudes (\Cref{sec:theory}),} and per-layer gradient measurements exhibit the predicted behavior in the trained models (\Cref{sec:gradstab}); together these attribute the state-space acceleration to the recurrence rather than to the readout. Complete specifications of the architecture, the parallel-scan and fused-kernel implementations, and all training logs accompany the paper, so that every reported quantity is re-derivable (\Cref{app:repro}). The paper is organized as follows. \Cref{sec:prelim} fixes notation and defines the three objects used throughout, complex states, unitary evolution, and the Born rule; it assumes no familiarity with quantum mechanics. \Cref{sec:wavefn} recalls the idealized substrate and its representational guarantee, and \Cref{sec:related} reviews the six lines of related work. \Cref{sec:substrate} specifies the deployable substrate one relaxation at a time and closes with a self-contained summary of the full construction. \Cref{sec:theory} states the two theoretical results: an optimization floor for rank-limited readouts and an exact-transport lemma for relative phase. \Cref{sec:setup,sec:results} describe the experimental design and the measurements, \Cref{sec:discussion} examines mechanism, cost, and scaling, and \Cref{sec:limitations} states each limitation together with the experiment that would remove it. Readers primarily interested in the empirical result may proceed directly to \Cref{sec:setup,sec:results}.

\begin{figure}[t]
\centering
\includegraphics[width=0.92\linewidth]{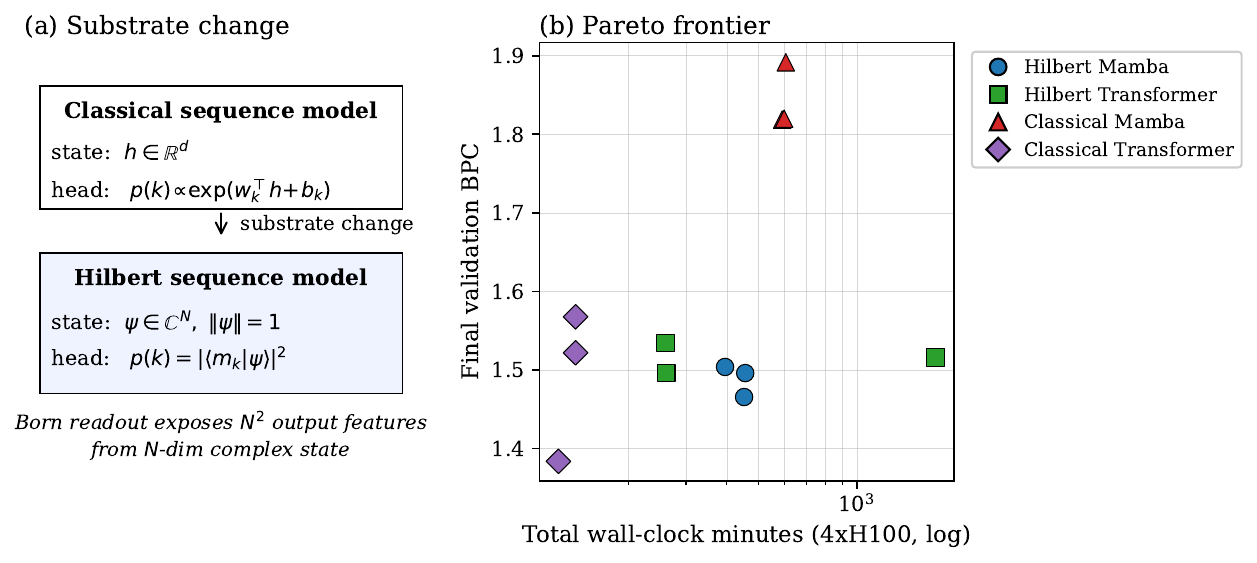}
\caption{\textbf{The substrate substitution and its effect.} Left: the deployable complex substrate replaces the real state and affine--softmax scoring of a classical block with a complex state, a diagonal Cayley phase recurrence, and Born-rule scoring; the rest of the block is unchanged. Right: validation bits per character against optimization step for the four models on three corpora. Within the measured range of targets ($2.5$ to $2.0$ bits per character), each complex variant (solid) reaches every value of the loss in one half to one third of the steps of its real counterpart (dashed), and the separation is established within the first $200$ steps. On enwik8, the smallest corpus, the complex models saturate the data and overfit beyond the budget, and the rotary Transformer baseline finishes below H-Transformer (\Cref{sec:smalldata}).}
\label{fig:hero}
\end{figure}
\section{Preliminaries and Notation}
\label{sec:prelim}
We fix notation and recall the algebraic objects used throughout. The material is standard and is developed at length in \citet{nebli2026quantum}; we restate only what makes the paper self-contained, and we define every object before using it. Readers familiar with Dirac notation and density matrices may proceed to \Cref{tab:correspondence}, which summarizes the correspondence between the idealized substrate of \Cref{sec:wavefn} and the deployable substrate of \Cref{sec:substrate}.

\paragraph{Complex states and the density matrix.}
The state space is the complex Hilbert space $\mathbb{C}^{N}$ with the Hermitian inner product $\braket{\phi}{\psi}=\sum_{k}\bar\phi_k\psi_k$. A state $\ket{\psi}$ is a column vector and $\bra{\psi}$ its conjugate transpose. Each state has a density matrix $\rho=\ket{\psi}\bra{\psi}\in\mathbb{C}^{N\times N}$: a rank-one, Hermitian, positive semidefinite operator with $\operatorname{tr}\rho=\|\psi\|^{2}$. Its diagonal entries $\rho_{kk}=|\psi_k|^{2}$ are per-coordinate occupation magnitudes. Its off-diagonal entries $\rho_{jk}=\bar\psi_j\psi_k$ carry the pairwise relative phases. The off-diagonals are the coordinates that a quadratic functional of the state can read and a linear one cannot, a fact that recurs throughout the paper.

\paragraph{Unitary evolution and the Cayley transform.}
An operator $U\in\mathbb{C}^{N\times N}$ is unitary if $U^\dagger U=I$; it then preserves inner products and norms, the complex analogue of a rotation. The Cayley transform parameterizes the unitary group by skew-Hermitian matrices ($A^\dagger=-A$),
\begin{equation}
U=(I-A)^{-1}(I+A),
\label{eq:cayley-full}
\end{equation}
{\color{black}and is defined and unitary for every skew-Hermitian $A$: the spectrum of $A$ is purely imaginary, so $I-A$ is always invertible (an exclusion condition arises only for the inverse map, which requires $-1$ not to be an eigenvalue of $U$).} For $A=-\tfrac{i}{2}H\,\Delta t$ with $H$ Hermitian, it coincides with the Crank--Nicolson (implicit midpoint) discretization~\citep{crank1947practical} of the Schr\"odinger evolution $\dot{\ket{\psi}}=-iH\ket{\psi}$. Unlike an explicit integrator, the Cayley transform preserves the norm exactly for any step size $\Delta t$; {\color{black}the phase trajectory is second-order accurate in $\Delta t$ (local error $O(\Delta t^{3})$), the standard accuracy of the implicit midpoint rule.} When $A$ is diagonal, the transform acts coordinatewise and maps each imaginary eigenvalue of $A$ onto the unit circle by a M\"obius map.

\paragraph{The Born rule.}
Given measurement vectors $\{\ket{e_v}\}_{v=1}^{V}\subset\mathbb{C}^{N}$, the Born rule~\citep{nielsen2000quantum} assigns outcome $v$ the probability $p_v=|\braket{e_v}{\psi}|^{2}=\operatorname{tr}\!\big(\ket{e_v}\bra{e_v}\,\rho\big)$. As a function of the state this is a real quadratic form; as a function of $\rho$ it is linear. That a linear functional of $\rho$ is a quadratic functional of $\ket{\psi}$ is the observation on which the separation of \Cref{sec:wavefn} rests. Throughout, $\sigma(\cdot)$ is the logistic function, $\odot$ the Hadamard product, and $\operatorname{tr}$ the trace; $T$ is the sequence length, $N$ the per-head state width, $H$ the number of heads, and $V$ the vocabulary size. These three objects, the complex state, the unitary map, and the quadratic measurement, are the entire mathematical inventory of the paper.

\begin{table}[t]
\centering
\caption{Correspondence between the idealized wave-function substrate of \citet{nebli2026quantum}, the deployable substrate of this paper, and their roles in the model.}
\label{tab:correspondence}
\small
\begin{tabular}{@{}lll@{}}
\toprule
Idealized substrate & Deployable substrate & Role \\
\midrule
Unit-norm state $\ket{\psi}\in\mathbb{C}^{N}$ & Complex state, norm free & Hidden state \\
Unitary $U_t=\mathrm{Cay}(-\tfrac{i}{2}H_t\Delta t)$ & Diagonal Cayley phase $\times$ gate & State transition \\
Norm preservation $\|\psi_t\|\equiv 1$ & Contraction by $\textstyle\prod_t\alpha_t$ & Forgetting and stability \\
Born vocabulary readout $|\braket{e_v}{\psi}|^{2}$ & Born attention score; tied softmax head & Output and scoring \\
Free Hamiltonian $H_0$ & Base frequencies $\omega_k$ & Baseline timescales \\
Interaction $H_{\mathrm{int}}(x_t)$ & Input-dependent $\delta_{t,k},\,\Delta t_{t,k}$ & Input coupling \\
\bottomrule
\end{tabular}
\end{table}

\section{The Wave-Function Substrate and Its Idealization}
\label{sec:wavefn}
\Cref{sec:intro} claimed that storing information in phase and reading it out quadratically yields a training advantage. Both ingredients were first assembled, in exact form, in the idealized model of \citet{nebli2026quantum}; the deployable substrate of this paper is a controlled relaxation of that model, and its relaxations are best understood against the model they relax. This section states the idealized substrate and the guarantee it enjoys, in the notation of \Cref{sec:prelim}, and refers to the cited work for constructions and proofs.

The idealized substrate couples three components. The state is a unit-norm vector $\ket{\psi_t}\in\mathbb{C}^{N}$. Its evolution is $\ket{\psi_{t+1}}=U_t\ket{\psi_t}$, where $U_t$ is generated by a learned Hermitian Hamiltonian $H_t=H_0+H_{\mathrm{int}}(x_t)$: a diagonal free term $H_0$ fixes a spectrum of baseline oscillation frequencies, and an input-dependent interaction $H_{\mathrm{int}}(x_t)$ couples the input to the dynamics. The Cayley discretization \eqref{eq:cayley-full} makes each $U_t$ exactly unitary, so $\|\psi_t\|\equiv 1$ at every step and for every sequence length. The readout is the Born rule over a learned measurement basis. Because $H_t$ is Hermitian by construction, norm preservation needs no penalties or projections, and the backward operator inherits the same unit modulus: gradients are neither amplified nor attenuated as they traverse the recurrence.

What distinguishes the idealized substrate is not norm preservation alone but the interaction between the complex state and the quadratic readout. {\color{black}Since $p_v=\operatorname{tr}(\ket{e_v}\bra{e_v}\rho)$ is linear in $\rho$, the family of Born readouts spans the full $N^{2}$-dimensional real vector space of Hermitian measurement functionals, whereas linear readouts of the state span only a $2N$-real-dimensional family. Two precisions apply: the pure-state image $\{\ket{\psi}\bra{\psi}\}$ is a variety of real dimension $2N-1$, so the $N^{2}$ entries of $\rho$ are algebraically dependent as features, and the count $N^{2}$ refers to independent measurement directions, not independent state coordinates.} \citet{nebli2026quantum} convert this dimensional gap into a lower bound. A family of phase-disambiguation tasks realized exactly by a complex unitary model of dimension $N$ requires state dimension $\Omega(N^{2})$ from any real model with orthogonal dynamics and an affine--softmax readout; the real model must spend an explicit state coordinate on each pairwise phase correlation. The bound is representational. It concerns what each class expresses at its optimum and says nothing about how the optimum is reached{\color{black}; moreover, it relies on assumptions, exact unitarity, a Born vocabulary readout, and orthogonal real dynamics, that the deployable model of this paper does not satisfy. We therefore use it strictly as motivation, and no part of the guarantee transfers to the architectures trained here.} This paper addresses the training side, and \Cref{sec:theory} supplies the corresponding optimization statement.

The mechanism the guarantee formalizes is interference, and the two-reading example of \Cref{sec:intro} is an instance of it; the word-level analogue, disambiguating an occurrence of ``bank'' by a later token, opens \citet{nebli2026quantum}. Two admissible readings of a context are encoded in two components of the state, and a later token is compatible with one reading but not the other. A real model suppresses the losing reading with a gate, a learned multiplicative mask that shrinks the disfavored component toward zero. The complex substrate suppresses it through the evolution itself: the recurrence advances the relative phase of the two components so that, under the quadratic readout, the losing reading interferes destructively with the incoming evidence while the winning one interferes constructively. Attenuation can only shrink a component, whereas interference can reverse its sign at the readout; the latter is a strictly larger repertoire of operations, available without a dedicated gating module.

With the idealized substrate in place, two questions arise: where this construction stands relative to prior work on complex, unitary, and quantum-inspired models, and what must change before it can train inside a production backbone. \Cref{sec:related} answers the first question, \Cref{sec:substrate} the second.

\section{Related Work}
\label{sec:related}
The substrate assembled in \Cref{sec:substrate} borrows from several literatures: its recurrence and scan come from the state-space line, its parameterization from the unitary line, its second host architecture from linear attention, its integrator from structure-preserving numerics, and its readout from quantum-inspired modeling. This section reviews each in enough detail to make the borrowings, and the departures, exact.

\subsection{Structured state-space models}
The state-space line grew out of the memory problem for recurrent networks. HiPPO~\citep{gu2020hippo} derived state matrices from online polynomial approximation of the input history, giving fixed-dimensional recurrences with provable memory of the past. S4~\citep{gu2021efficiently} made these systems trainable at depth by parameterizing a stable linear time-invariant system and evaluating its sequence map as a convolution, setting the standard on long-range benchmarks such as the Long Range Arena~\citep{tay2021long}. Two simplifications followed. DSS~\citep{gupta2022diagonal} and S4D~\citep{gu2022parameterization} showed that diagonal state matrices recover the performance of the full structured system once the initialization is chosen correctly, and S5~\citep{smith2023simplified} showed that a diagonal recurrence evaluated by an associative parallel scan suffices, at a fraction of the implementation complexity. The linear recurrent unit~\citep{orvieto2023resurrecting} completed this reduction by asking which ingredients are essential: it parameterizes complex-diagonal eigenvalues by log-magnitude and phase, initializes them on a ring close to the unit circle, and identifies the input normalization that makes deep stacks of bare linear recurrences trainable.

A second thread carried the line to language. H3~\citep{fu2023hungry} diagnosed the gap between state-space models and attention on associative-recall tasks and closed part of it with a hybrid design; Hyena~\citep{poli2023hyena} replaced attention with long implicit convolutions; Mamba~\citep{gu2023mamba} made the recurrence parameters functions of the current token, so that the model can selectively write, retain, or discard, and paired this selective recurrence with a hardware-aware scan, reporting language modeling competitive with Transformers at linear-time inference. Mamba-2~\citep{dao2024mamba2} exhibited the selective recurrence as a form of structured masked attention, unifying the two families.

The recurrence \eqref{eq:rec} is a member of this family: a diagonal, input-dependent linear recurrence, evaluated by the same associative scan (\Cref{sec:scan}), hosted in Mamba's own block (\Cref{sec:realizations}). The departure is confined to a single design decision. Every model above keeps its eigenvalues strictly inside the unit disc, so stability, forgetting, and memory all hang on one quantity, $|\lambda|<1$: the recurrence forgets by attenuating, and the same attenuation erodes what it remembers. The tension is visible within the line itself, in the LRU's finding that eigenvalues must crowd the unit circle for long memory yet stay inside it for stability. The substrate studied here dissolves the tension by splitting the roles: the Cayley map holds the phase of every eigenvalue exactly on the circle, a separate sigmoid gate supplies the decay, and the information rides in the phase, which the decay never touches. The remaining difference is the readout: the models above read their state linearly, which leaves relative phases unobservable, whereas the Born score of \Cref{sec:readout} reads them directly. H-Mamba is Mamba with this substitution and nothing else.

\subsection{Complex-valued and unitary networks}
Complex-valued networks predate deep learning; the signal-processing and neural-network literature developed complex arithmetic, activation functions, and training rules for them over several decades~\citep{mandic2009complex,hirose2012complex}. Within deep learning, the unitary RNN~\citep{arjovsky2016unitary} introduced complex recurrences to solve a conditioning problem: constraining the recurrent operator to a product of structured unitary factors (diagonal phases, Householder reflections, Fourier transforms) fixes the spectral norm of the recurrent Jacobian at one and eliminates exploding and vanishing gradients by construction. EUNN~\citep{jing2017tunable} made the parameterization tunable in capacity, the full-capacity uRNN~\citep{wisdom2016full} optimized over the entire unitary group, scoRNN~\citep{helfrich2018orthogonal} obtained orthogonal recurrent matrices through a scaled Cayley transform, and \citet{lezcano2019cheap} gave exact, inexpensive trivializations of the orthogonal and unitary groups; geometric numerical integration~\citep{hairer2006geometric} supplies the discretizations that keep learned dynamics on the group. Deep complex networks~\citep{trabelsi2018deep} extended convolution, batch normalization, and initialization to $\mathbb{C}$, complex gated RNNs~\citep{wolter2018complex} combined complex arithmetic with gating, and complex attention has been developed for inherently complex-valued data such as audio and radar~\citep{yang2020complex}; tensor-network models~\citep{stoudenmire2016supervised} apply the adjacent multilinear algebra to supervised learning.

This line settled two questions on which our construction depends: complex arithmetic trains stably at depth, and unitarity removes gradient pathologies. It also left two questions open, and they define what we do here. First, exact unitarity forbids forgetting, and the models of this line were accordingly validated on memory benchmarks, copying, adding, and pixel-by-pixel classification, rather than on competitive language modeling; the constraint that stabilizes gradients also freezes stale state. The substrate here keeps the unit-modulus phase but moves the modulus into a separate learned gate, so that forgetting returns without sacrificing the transport property, which \Cref{lem:phase} states exactly. Second, every model in this line reads its state through a linear map, {\color{black}so the $\Theta(N^2)$ pairwise phase correlations a complex state carries are not resolved at the output;} the complex structure works as a conditioning device, not as an information channel. Even the Cayley transform, where it appears~\citep{helfrich2018orthogonal,lezcano2019cheap}, serves to keep a real recurrent matrix orthogonal; here the same map is applied coordinatewise to keep individual phases on the circle while the Born score turns those phases into predictions.

\subsection{Linear attention and relative position encodings}
Linear attention~\citep{katharopoulos2020transformers} replaces the softmax kernel with a positive feature map, which factorizes the attention matrix and turns autoregressive attention into a linear recurrence with $O(T)$ inference; Performers~\citep{choromanski2021rethinking} construct random features that approximate the softmax kernel without bias. A subsequent generation attached forgetting to this recurrence: RetNet~\citep{sun2023retentive} with a fixed exponential decay, RWKV~\citep{peng2023rwkv} with an attention-free mixing rule in the same spirit, and gated linear attention~\citep{yang2024gla} with data-dependent gates and a hardware-efficient training algorithm. In parallel, rotary position embeddings~\citep{su2021roformer} became the standard way to encode relative position: queries and keys are rotated in two-dimensional subspaces through angles proportional to absolute position, so that their inner product depends on position only through differences. A rotation of a two-dimensional subspace is multiplication by a unit-modulus complex number, so every model with rotary embeddings already stores position in phase.

These two threads each contain half of our attention realization. The gated linear-attention thread maintains a recurrent state with a learned decay gate, exactly the role the gate $\alpha$ plays in \eqref{eq:rec}, but scores with nonnegative feature maps, which can only accumulate evidence. The rotary thread stores information in phase, but with frequencies fixed by schedule, phases that encode position alone, and a score that is, in complex notation, the real part of a complex inner product, a signed but linear functional of the query--key overlaps. Neither thread scores with a quadratic form, so in neither can one overlap cancel another. The realization of \Cref{sec:readout} combines the halves and adds the missing nonlinearity: learned frequencies whose phases participate in the computation, a decay gate for forgetting, and the squared-modulus score $|\braket{q_i}{k_j}|^{2}$, {\color{black}whose cross terms make interference available in the attention weights themselves, at a strength scaled by the participating magnitudes.}

\subsection{Hamiltonian, oscillatory, and energy-preserving networks}
A fourth line builds conservation and oscillation into network dynamics. Hamiltonian neural networks~\citep{greydanus2019hamiltonian} learn a scalar energy whose induced flow the model follows, conserving the learned energy along trajectories; symplectic recurrent networks~\citep{chen2020symplectic} integrate learned separable Hamiltonians with leapfrog integrators and show that the integrator, not the Hamiltonian form alone, secures stable learned dynamics. On the recurrent side, \citet{haber2017stable} pose forward propagation as a well-posed differential equation whose stability is controlled through the spectrum of the transition operator; AntisymmetricRNN~\citep{chang2019antisymmetric} instantiates this with an antisymmetric weight matrix, whose purely imaginary spectrum is the real-valued analogue of unitary evolution; and coRNN~\citep{rusch2021coupled} builds the recurrence from a network of controlled, damped oscillators, with proven bounds on gradient growth and strong results on long-dependency benchmarks.

The oscillatory members of this line come closest to the substrate studied here: coRNN, like the recurrence here, uses oscillation as the vehicle of memory, and our per-channel frequencies $\omega_k$ have their counterpart in its oscillator frequencies. The difference is what a real oscillator can carry. A real oscillation stores its state in a position--velocity pair whose amplitude decays under damping, so information and decay share coordinates, and a linear readout of those coordinates sees amplitudes, not phase relations. A complex coordinate is the canonical pair packaged as one number: the gate damps its modulus while the argument advances undisturbed, and the quadratic readout of \Cref{sec:readout} converts the arguments' differences into probabilities. Conservation likewise changes role. In this line the conserved quantity is a real energy or norm, conserved as physical fidelity or as a stability device; here the conserved quantity is the phase alone, the norm is deliberately not conserved so that the model can forget, and the conserved quantity is what the model predicts with.

\subsection{Continuous-time and quantum-inspired models}
Neural ordinary differential equations~\citep{chen2018neural} formalize deep networks as discretizations of continuous-time systems and make explicit that a learned dynamical system is shaped jointly by its vector field and its integrator. That lesson fixes a choice made here: the Crank--Nicolson discretization~\citep{crank1947practical} of \Cref{sec:prelim} preserves the unitary structure of the continuous flow exactly, at every step size, where an explicit integrator would preserve it only to first order.

Quantum-inspired modeling of language predates deep learning as well. Quantum language models for information retrieval~\citep{sordoni2013modeling} represented queries and documents as density matrices and scored them with Born-type traces, capturing term dependencies that bag-of-words models miss; this is, to our knowledge, the earliest use in language processing of the density-matrix lift that \Cref{sec:wavefn} exploits. Quantum cognition~\citep{busemeyer2012quantum} documented systematic violations of classical probability in human judgment that quantum probability describes, motivating interference as an inductive bias for ambiguity; and a separate engineering line develops variational circuits for quantum hardware~\citep{preskill2018quantum}, which shares formalism but not purpose with classical quantum-inspired models. The framework this paper builds on~\citep{nebli2026quantum} belongs to the classical line and assembled the full package for sequence modeling: a complex state, norm-preserving Hamiltonian evolution, a Born-rule readout, the $\Omega(N^{2})$ separation theorem recalled in \Cref{sec:wavefn}, and conserved probability currents for tracing information flow. What that work left open, and what we supply here, is the training side: whether the substrate, once relaxed enough to live inside production backbones, changes how fast such models learn. Everything from \Cref{sec:substrate} onward addresses that question at 253M parameters.

\subsection{The softmax bottleneck}
\citet{yang2018breaking} proved that an affine--softmax head of width $d$ produces log-probability matrices of rank at most $d+1$, so a model whose true conditional structure has higher rank is misspecified no matter how its backbone computes the state; their mixture of softmaxes restores rank by combining $K$ softmax components at a $K$-fold readout cost. Follow-up work varied the fix while keeping the diagnosis: Sigsoftmax~\citep{kanai2018sigsoftmax} modifies the output nonlinearity, and Mixtape~\citep{yang2019mixtape} replaces the mixture with an efficient gating construction. All of these remedies operate on a real state and raise rank by complicating the output function.

In the vocabulary of this paper, the bottleneck is a substrate property, the first of the two limitations with which \Cref{sec:intro} opened, and the line bears on this paper twice. {\color{black}\Cref{prop:gap} sharpens the diagnosis from expressibility to optimization: where the head is narrower than the rank of the target, no training trajectory through an affine--softmax head reaches the optimum, by a margin the tail spectrum quantifies. That condition is not met at the byte vocabulary of our experiments (\Cref{sec:floor}), so the sharpened statement is offered as general motivation.} And the mixture of softmaxes supplies the natural control for our readout mechanism: it raises rank without complex numbers, so if it closed the gap we measure on the state-space backbone, the readout rather than the recurrence would be implicated; \Cref{sec:mechanism} places this control alongside the factorial ablation. The Born readout of the idealized substrate raises the accessible rank to order $N^{2}$ by changing the functional form of the readout instead of mixing copies of it.

Across these six literatures, the individual ingredients of the deployable substrate all appear: diagonal scans, Cayley parameterizations, decay gates, phase-encoded structure, density-matrix readouts, and rank-raising output layers. Their combination, a contractive recurrence whose eigenvalue phases are exactly unit-modulus, feeding a quadratic score, at parameters matched to a production baseline, does not, and neither does a measurement of what that combination changes about optimization. Those are the subjects of the next four sections.

\section{The Deployable Complex Substrate}
\label{sec:substrate}
Embedding the idealized substrate in a contemporary backbone requires relaxing its two most restrictive properties: exact unitarity and the Born vocabulary readout. Each is in tension with a mechanism that trained sequence models rely on. This section specifies the relaxed substrate and states, for each relaxation, what is surrendered and what is retained. The substrate is defined per head; a model uses $H=16$ heads of width $N=16$, and the two architectural realizations that host it are given in \Cref{sec:realizations}. The subsections follow the order of a forward pass, from the state through its evolution to the readout, and \Cref{sec:cost} closes with a summary of the full specification.

\subsection{State}
\label{sec:state}
In place of the real hidden vector $\mathbf{h}_t\in\mathbb{R}^{d}$, each head carries a complex vector $\ket{\psi_t}\in\mathbb{C}^{N}$, stored as real and imaginary parts: $2N$ real numbers per head. The state is initialized from a learned complex prior normalized to unit modulus, and the choice is deliberate. Were the state initialized on the real axis, every off-diagonal entry $\rho_{jk}=\bar\psi_j\psi_k$ of the density matrix would be real at the first step. The interference cross terms would vanish identically, and early optimization would be spent rotating the state off the real axis before any phase-dependent computation could begin. A complex prior places nontrivial phase in the state from the first token. The idealized substrate keeps $\|\psi_t\|=1$ at every step; the deployable substrate normalizes only the prior and lets the norm vary thereafter. That is the first relaxation, and it is the subject of the next subsection.

\subsection{Evolution and the relaxation of unitarity}
\label{sec:evolution}
The state evolves by a diagonal, input-dependent linear recurrence,
\begin{equation}
\ket{\psi_{t+1}}=a_t\odot\ket{\psi_t}+b_t,\qquad a_t=\alpha_t\odot\lambda_t\in\mathbb{C}^{N},
\label{eq:rec}
\end{equation}
where $b_t\in\mathbb{C}^{N}$ is a complex input projection and each transition coefficient $a_{t,k}$ is a unit-modulus phase $\lambda_{t,k}$ times a real gate $\alpha_{t,k}\in(0,1)$. The phase is the diagonal Cayley transform of a real parameter $\phi_{t,k}$,
\begin{equation}
\lambda_{t,k}=\frac{1+i\phi_{t,k}}{1-i\phi_{t,k}},\qquad |\lambda_{t,k}|=1,\qquad \arg\lambda_{t,k}=2\arctan\phi_{t,k},
\label{eq:cayley}
\end{equation}
so without the gate the transition spectrum lies on the unit circle. The phase parameter is assembled from an input-dependent step size $\Delta t_{t,k}=\operatorname{softplus}(\cdot)$, an input-dependent shift $\delta_{t,k}$, and a per-channel base frequency $\omega_k$, as $\phi_{t,k}=\tfrac{1}{2}\Delta t_{t,k}(\delta_{t,k}+\omega_k)$. This reproduces, coordinatewise, the free-plus-interaction decomposition of the idealized Hamiltonian: $\omega_k$ supplies fixed baseline frequencies, $\delta_{t,k}$ the input-dependent modulation.

The exactness of the parameterization matters at this sequence length, and it is the reason we prefer the Cayley map to approximate norm control by spectral normalization, penalty terms, or projection steps. The Cayley map holds $|\lambda_{t,k}|=1$ at machine precision with two arithmetic operations per coordinate and nothing to tune. A scheme that controls the modulus only approximately, to within $10^{-3}$ say, accumulates a modulus drift of $0.999^{1024}\approx 0.36$ over the sequence, losing $64\%$ of the transported signal to the attenuation the phase channel is designed to remove. Approximate orthogonality suffices as a stability device; {\color{black}exact unit modulus is what makes the phase transport of \Cref{lem:phase} distortion-free.}

The gate $\alpha_{t,k}=\sigma(\cdot)$ constitutes the relaxation of unitarity. It sets the modulus of the transition coefficient, $|a_{t,k}|=\alpha_{t,k}$, and hence the rate at which coordinate $k$ forgets its history: a value near one preserves the coordinate, a value near zero discards it. This restores the selective, input-dependent forgetting on which gated state-space models depend~\citep{gu2023mamba} and which an exactly unitary recurrence cannot express; in the example of \Cref{sec:intro}, it is the gate that allows the model to discard both readings once the passage has moved on. The cost of the relaxation is that the recurrence becomes contractive, and the exact gradient guarantee of the idealized substrate is forfeited. Two properties are retained, and they are the reason the relaxation is admissible. {\color{black}The phase $\lambda_{t,k}$ remains on the unit circle, so the argument of a coordinate is transported without distortion even while its magnitude decays; and the sign of interference is unaffected, because constructive against destructive is decided by relative phase, although the strength of the cross term scales with the participating magnitudes.} The residual gradient benefit this predicts is formalized in \Cref{lem:phase} and measured in \Cref{sec:gradstab}.

The contrast that motivates the phase channel is quantitative. A real diagonal recurrence with per-coordinate factor $\rho<1$ transports a signal across $T$ steps with gain $\rho^{T}$; at $\rho=0.99$ and $T=1024$ this gain is $3.4\times 10^{-5}$. {\color{black}The complex recurrence is likewise contractive, with magnitude gain $\prod_{t}\alpha_{t,k}$, but the phase coordinate itself is transported without distortion, so the information encoded in the argument arrives undistorted at the end of the sequence; its observable strength there is set by the accumulated magnitudes (\Cref{rem:lemscope}).} The complex recurrence does not avoid contraction; it separates a decaying channel, the magnitude, from a preserved channel, the phase, and places the transported information in the preserved one.

\begin{remark}[Interaction picture and effective dependency length]
\label{rem:interaction}
The decomposition $\phi_{t,k}=\tfrac12\Delta t_{t,k}(\delta_{t,k}+\omega_k)$ separates a fixed baseline rotation ($\omega_k$) from an input-dependent modulation ($\delta_{t,k}$). Transforming into the frame co-rotating with the baseline, the interaction picture of \citet{nebli2026quantum}, removes the baseline rotation and aligns positions that differ only by baseline phase. {\color{black}This suggests a third mechanism for the per-step advantage, complementary to the score's interaction rule (\Cref{sec:theory}) and gradient propagation (\Cref{sec:gradstab})}: absorbing the predictable part of the evolution into a change of frame shortens the effective length over which the input-dependent dynamics must be resolved. We record the hypothesis without testing it here.
\end{remark}

\subsection{Evaluation by parallel scan}
\label{sec:scan}
Equation~\eqref{eq:rec} is a first-order linear recurrence and admits an associative-scan solution. Writing the transition in polar form and accumulating over a chunk of length $C$,
\begin{equation}
\ket{\psi_{s+m}}=P_m\odot\Big(\ket{\psi_s}+\sum_{r=1}^{m}P_r^{-1}\odot b_{s+r}\Big),\qquad P_m=\prod_{r=1}^{m}a_{s+r},
\label{eq:scanform}
\end{equation}
where the cumulative products and sums are prefix scans over the chunk, and the final state of each chunk seeds the next. The cumulative product is evaluated in the logarithmic domain, as a prefix sum of $\log|a_{t,k}|$ clamped from below so that the reciprocal $P_r^{-1}$ cannot overflow, together with a prefix sum of $\arg a_{t,k}$, and is reassembled as a magnitude at an angle. The construction computes the exact trajectory in $O(T)$ arithmetic with $O(\log C)$ sequential depth. Because the recurrence is diagonal, each step costs $O(N)$, not the $O(N^{2})$ of a dense transition. The substrate is therefore cheaper per step than the dense real recurrence it replaces, and complex arithmetic enters only as a constant factor on the $O(N)$ term.

\subsection{Relation to real diagonal state-space models}
\label{sec:relssm}
Structurally, the recurrence \eqref{eq:rec} is a diagonal linear state-space model of the kind studied by S5~\citep{smith2023simplified} and the linear recurrent unit~\citep{orvieto2023resurrecting}, and it is worth stating exactly where it departs from them. A real diagonal model places each transition eigenvalue on the real line or inside the unit disc, and the information carried across time sits in the \emph{magnitude} of the state, which the eigenvalue attenuates geometrically. Our recurrence places the phase of each eigenvalue on the unit circle by \eqref{eq:cayley}, {\color{black}and the carried information sits in the \emph{argument} of the state, which the eigenvalue rotates without distortion; the magnitude, which sets how strongly that argument is observed downstream, is governed separately by the gate, whose sole role is forgetting.} The two designs make opposite choices about which component of a complex number transports and which forgets. The real diagonal model transports in magnitude and cannot forget without attenuating the transported signal. Ours transports in phase and forgets in magnitude, so transport and forgetting are decoupled. The gradient measurements of \Cref{sec:gradstab} are the empirical trace of this decoupling.

\subsection{Readout and the relaxation of the Born vocabulary head}
\label{sec:readout}
The second relaxation concerns the Born rule. In the idealized substrate the Born rule is the vocabulary readout, and its density-matrix lift is the source of the $\Omega(N^{2})$ advantage. But a Born vocabulary head evaluates $|\braket{e_v}{\psi}|^{2}$ against one measurement vector per token, costing $O(NV)$ per step, and realizing the full rank advantage requires $N=\Omega(\sqrt{V})$. At subword vocabularies the state width and readout cost would grow as $\sqrt{V}$ and $V$. Rather than pay this cost before the substrate has proven useful, we apply the Born rule where it is inexpensive, as the scoring function of a linear-attention kernel, and keep a conventional tied-softmax vocabulary head. The two backbones consequently read the complex state differently.

The choice of the squared modulus over some other nonlinear projection is likewise not arbitrary. The Born score is the rank-one case of the general quadratic readout $\operatorname{tr}(M\rho)$ with $M$ positive semidefinite: it is nonnegative, it is normalizable without exponentials, which is what permits the running-sum attention normalization in \eqref{eq:born-attn}, and it is linear in the density matrix $\rho$. The linearity in $\rho$ is the property on which the representational separation of \Cref{sec:wavefn} and the rank accounting of \Cref{sec:theory} rest; a generic nonlinearity of the state would surrender both the probabilistic interpretation and the analysis. The remaining alternatives, a complex recurrence read by a plain softmax and a real recurrence read by a Born score, are exactly the factorial cells specified in \Cref{sec:mechanism}.

In the state-space realization the recurrence output is $y_t=\bar C_t\odot\ket{\psi_t}$ for a learned complex projection $C_t$; real and imaginary parts are concatenated, mapped to the inner width, gated by a SiLU branch, and added to the residual stream, from which the tied-softmax head produces the distribution. The Born rule does not enter this readout at all: the complex substrate acts on the state-space backbone entirely through the recurrence. The attention realization scores position $i$ against position $j$ by
\begin{equation}
s_{ij}=\frac{|\braket{q_i}{k_j}|^{2}}{D\,\tau_i},\qquad \ket{\psi_i}=\frac{\sum_{j\le i}s_{ij}\ket{v_j}}{\sum_{j\le i}s_{ij}},
\label{eq:born-attn}
\end{equation}
a causal, sum-normalized attention whose weight is the squared modulus of a complex inner product rather than the softmax of a real dot product, with temperature $\tau_i$ and head width $D=N$. Queries and keys are complex and are rotated before scoring by a per-position phase $\theta_{t,k}=t\,\omega_k$ with learnable base frequencies, a complex relative encoding in the rotary family. Because the weight is a squared modulus, {\color{black}overlaps between query and key components reinforce or cancel according to their relative phase, with a cross-term strength proportional to the participating magnitudes $|\braket{q_i}{k_j}|$; the Born rule thus supplies interference at the level of attention even though the vocabulary readout remains a tied softmax.}

A two-state computation exhibits the mechanism shared by both readouts and formalizes the example of \Cref{sec:intro}. Let $\ket{\phi_1},\ket{\phi_2}$ be orthonormal states encoding the two readings and let $\ket{\psi}=\tfrac{1}{\sqrt 2}\big(\ket{\phi_1}+e^{i\theta}\ket{\phi_2}\big)$. A measurement vector $\ket{e}=\tfrac{1}{\sqrt 2}\big(\ket{\phi_1}+\ket{\phi_2}\big)$ aligned with their sum returns
\[
p=|\braket{e}{\psi}|^{2}=\tfrac12\,(1+\cos\theta),
\]
which falls from $1$ at $\theta=0$ to $0$ at $\theta=\pi$ as the relative phase alone advances, both magnitudes held fixed. The recurrence controls $\theta$ through the Cayley phase, so a downstream token can extinguish one reading by rotating it into phase opposition. A real readout, which does not represent $\theta$, can only attenuate a component toward zero; it cannot make a present component count with negative sign. This is the operational content of the separation recalled in \Cref{sec:wavefn}.

\subsection{Initialization}
\label{sec:init}
Initialization places nontrivial phase in the computation from the first step and spreads the per-channel frequencies across timescales. The complex prior and the complex measurement vectors of the attention readout are drawn from a circularly symmetric complex normal distribution and normalized to unit modulus; a real initialization would zero every interference cross term at the first step (\Cref{sec:state}) and empirically delays the onset of learning. The base frequencies $\omega_k$ are initialized on a geometric grid spanning roughly two decades, with alternating sign across channels, so the substrate starts with a spectrum of oscillation timescales instead of one characteristic rate; this mirrors the diagonal free Hamiltonian $H_0$ of the idealized model. Input projections use the Xavier rule, and the per-head parameter projection starts with a small standard deviation, so the input-dependent modulation of phase and gate begins near zero and grows over training. These choices are held fixed across all models and corpora.

\subsection{Architectural realizations}
\label{sec:realizations}
Both complex realizations inherit the block structure of Mamba and modify only the substrate. The shared block applies RMS normalization; an input projection $\mathbb{R}^{d}\to\mathbb{R}^{2d_{\mathrm{inner}}}$ that separates a signal branch from a SiLU gate branch; a depthwise causal convolution of kernel width 4 with a SiLU nonlinearity on the signal branch; a projection to the per-head substrate parameters; the substrate itself; a projection back to the inner width; the gate; and an output projection with a residual connection. The input embedding is tied to the output head. All models share one configuration: $d_{\mathrm{model}}=1024$, $d_{\mathrm{inner}}=2048$, $L=24$ layers, $H=16$ heads, $N=16$, convolution width 4, sequence length $1024$, byte vocabulary $256$, and 253M total parameters (\Cref{tab:params}).

The state-space realization, H-Mamba (H for the Hilbert-space state), obtains its substrate parameters from a per-head projection emitting $6N+1$ real numbers: the phase shift $\delta$ ($N$ values), a logit for the step size $\Delta t$ (1 value), a logit for the gate $\alpha$ ($N$ values), and the complex input and output projections $b$ and $C$ ($2N$ values each). The step size is the $\operatorname{softplus}$ of its logit, clamped at $0.1$; the gate is the logistic of its logit; the phase follows \eqref{eq:cayley}; and the recurrence \eqref{eq:rec} is evaluated by the chunked scan of \Cref{sec:scan} with chunk length 64. Every component downstream of the recurrence (the gate, the residual connection, and the tied-softmax head) is identical to that of the real Mamba.

The attention realization, H-Transformer, obtains complex queries, keys, values, and a temperature from the same per-head projection, hosted in the same shared block. Queries and keys are phase-rotated, scored by \eqref{eq:born-attn} under the causal running-sum normalization, and read out bilinearly against the query into the same tied-softmax head. The attention is computed by a single fused kernel, implemented in Triton with a compiled fallback, that evaluates the causal Born attention in one pass with single-precision accumulation; this lets the complex attention train in half precision without instability. Its real counterpart is the standard pre-norm Transformer block, rotary dot-product attention~\citep{su2021roformer} followed by a SiLU feed-forward layer, at its published configuration.

\paragraph{Parameter matching.}
For the comparison to isolate the substrate, each complex model and its real counterpart must have equal capacity. We match each pair tensor by tensor to within $0.02\%$ of the total, letting the real baselines absorb the small surplus introduced by the complex projections in their feed-forward capacity. Concretely, each real-Mamba block carries a small SwiGLU branch of hidden width $1258$ whose only function is parameter matching, and the real Transformer's feed-forward hidden width is $3096$; each value is chosen so that the pair totals 253M. {\color{black}No other dimension differs, so the results below are not attributable to a difference in total parameter count, though the two members of a pair allocate those parameters differently.}

\subsection{Parameter and computational cost}
\label{sec:cost}
The substrate is inexpensive relative to the block that hosts it. Per head, the recurrence adds the $6N+1$ parameters of the per-head projection and runs the scan in $O(N)$ arithmetic per step, against $O(N^{2})$ for a dense real transition of the same width. The complex state doubles state storage from $N$ to $2N$ reals per head, a negligible fraction of block memory, which the shared input and output projections dominate. \Cref{tab:cost} summarizes per-component parameters and per-step arithmetic. Because the extra arithmetic of complex numbers is a constant factor on a non-dominant term, the complex models are not materially more expensive per step than the real ones; the measured throughput of \Cref{sec:throughput} is consistent with this accounting.

\paragraph{Summary of the deployable substrate.}
Each head carries a complex state of width $N=16$, stored as $2N$ real numbers and initialized from a unit-modulus complex prior. The state evolves by the diagonal recurrence \eqref{eq:rec}. Each transition coefficient is a unit-modulus Cayley phase \eqref{eq:cayley}, which carries the information, multiplied by a sigmoid gate, which does the forgetting; the recurrence is evaluated exactly by a chunked scan in $O(N)$ arithmetic per step. The attention realization scores positions by the Born rule \eqref{eq:born-attn}; the state-space realization does not use the Born rule at all. Both realizations end in an ordinary tied-softmax head. Relative to the idealized substrate, two properties are surrendered: exact norm preservation, since the recurrence contracts, and the $\Omega(N^{2})$ vocabulary readout, since the head is a softmax. Two properties are retained: the unit-modulus phase and the interference it supports, and \Cref{lem:phase} will make the retention exact. Relative to the real baselines, the only changes are the complex state, the phase--gate transition, and, in the attention model, the Born score; parameters match to within $0.02\%$. \Cref{sec:theory} gives the two retained properties their formal statements.

\begin{table}[t]
\centering
\caption{Per-head parameters and per-step arithmetic of the substrate and the shared block components. Here $N=16$ is the state width, $d_{\mathrm{inner}}=2048$ the inner width, $d_{\mathrm{model}}=1024$, $V=256$ the vocabulary, and $T$ the sequence length.}
\label{tab:cost}
\small
\begin{tabular}{@{}lll@{}}
\toprule
Component & Parameters (per head) & Per-step arithmetic \\
\midrule
Dense real transition & $O(N^{2})$ & $O(N^{2})$ \\
Diagonal Cayley recurrence & $6N+1$ & $O(N)$ \\
Born attention score & shared with $q,k,v$ & $O(NT)$ per query \\
Input / output projections (shared) & $O(d_{\mathrm{inner}})$ & $O(d_{\mathrm{inner}})$ \\
Tied vocabulary head (shared) & tied to embedding & $O(d_{\mathrm{model}}V)$ \\
\bottomrule
\end{tabular}
\end{table}

\begin{figure}[t]
\centering
\includegraphics[width=\linewidth]{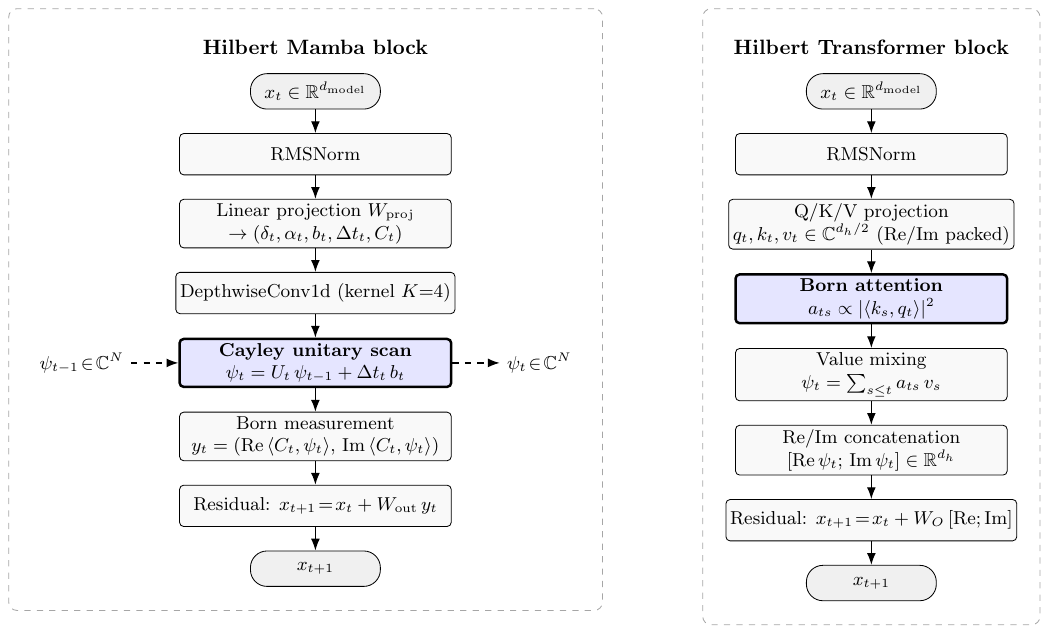}
\caption{\textbf{The substitution is local.} The state-space realization (left) and the attention realization (right) share the block skeleton of the classical backbone (unshaded): convolution branch, gate branch, residual stream, tied output head. The shaded components are the substrate: the complex state, the diagonal Cayley phase recurrence, and, in the attention realization, Born-rule scoring. Every unshaded tensor matches its real-valued counterpart in shape and initialization.}
\label{fig:method}
\end{figure}

\begin{table}[t]
\centering
\caption{Configuration shared by all four models. Real and complex counterparts are matched to within $0.02\%$ of parameters; the real baselines absorb the difference in feed-forward capacity (a parameter-matching SwiGLU branch of hidden width $1258$ in each Mamba block; feed-forward hidden width $3096$ in the Transformer).}
\label{tab:params}
\small
\begin{tabular}{@{}ll@{}}
\toprule
Hyperparameter & Value \\
\midrule
Layers $L$ & 24 \\
Model dimension $d_{\mathrm{model}}$ & 1024 \\
Inner dimension $d_{\mathrm{inner}}$ & 2048 \\
Heads $H$ / state width $N$ & 16 / 16 \\
Sequence length $T$ & 1024 \\
Vocabulary $V$ & 256 (byte-level) \\
Total parameters & 253M \\
\bottomrule
\end{tabular}
\end{table}
\section{Theoretical Analysis}
\label{sec:theory}
The separation recalled in \Cref{sec:wavefn} is about representation: it bounds the state dimension a real model needs to \emph{express} a task family. Our measurements are about optimization: they count the gradient steps a model needs to \emph{reach} a loss. This section provides two statements that connect the two. {\color{black}The first is a floor on the loss attainable by a rank-limited readout. It is stated for general model classes; \Cref{sec:floor} shows that it is inactive at the vocabulary head of every model trained here, so it motivates the present results without explaining them.} The second is an exact-transport lemma for relative phase, which formalizes the advantage available to the recurrence and predicts the gradient measurements of \Cref{sec:gradstab}.

\subsection{An optimization floor for the affine--softmax readout}
\label{sec:floor}
The floor rests on a simple observation: an affine--softmax head produces only logit matrices of low rank, so the set of predictors it can ever visit is a thin slice of the space in which the target lies. If the target carries spectral mass outside that slice, no trajectory through the slice approaches it; the loss is bounded away from the optimum at every step by an amount proportional to that mass, and per-step progress must stall as the trajectory nears the bound. The proposition below makes this quantitative, and its proof is elementary.

Fix a conditioning distribution supported on $M$ contexts, taken uniform for simplicity, and let $L^{\star}\in\mathbb{R}^{M\times V}$ be the matrix whose row $c$ holds the optimal next-token logits for context $c$. Because the softmax is invariant to adding a constant to a row, we pass to row-centered representatives: for any $Z\in\mathbb{R}^{M\times V}$ write $\widetilde{Z}=Z(I-\tfrac1V\mathbf{1}\mathbf{1}^{\top})$, which projects each row orthogonal to the all-ones vector. An affine--softmax head of width $d$ produces stacked logits $Z=HW^{\top}+\mathbf{1}b^{\top}$ with $H\in\mathbb{R}^{M\times d}$, so $\operatorname{rank}(Z)\le d+1$~\citep{yang2018breaking}, and centering cannot raise rank: $\operatorname{rank}(\widetilde{Z})\le d+1$ for every parameter setting. Write $\sigma_1\ge\sigma_2\ge\cdots$ for the singular values of $\widetilde{L}^{\star}$, and let $p_{\min}>0$ be a floor such that every conditional probability produced along the segment between the realized and the optimal logits is at least $p_{\min}$ (an assumption about the region training traverses).

\begin{proposition}[Optimization floor for the affine--softmax readout]
\label{prop:gap}
Under the assumptions above, for every parameter setting of the backbone and head, and hence after any number of gradient steps, the expected excess cross-entropy obeys
\begin{equation}
\mathbb{E}_{c}\!\left[\mathcal{L}_c-\mathcal{L}_c^{\star}\right]\;\ge\;\frac{p_{\min}}{2M}\sum_{i>d+1}\sigma_i^{2}\big(\widetilde{L}^{\star}\big).
\label{eq:gap}
\end{equation}
The floor is strictly positive whenever $\operatorname{rank}(\widetilde{L}^{\star})>d+1$ and vanishes otherwise.
\end{proposition}

\begin{proof}
The excess loss at context $c$ is $\mathrm{KL}(p^{\star}_c\,\|\,\hat p_c)$. With $A(z)=\log\sum_v e^{z_v}$, a direct computation gives $\mathrm{KL}(p^{\star}_c\|\hat p_c)=A(\hat z_c)-A(z^{\star}_c)-\langle\nabla A(z^{\star}_c),\hat z_c-z^{\star}_c\rangle$, the Bregman divergence of $A$; both sides are invariant to row shifts, so take $\Delta_c=\widetilde{\hat z}_c-\widetilde{z}^{\star}_c$. Taylor's theorem with integral remainder yields
$\mathrm{KL}=\int_0^1(1-s)\,\Delta_c^{\top}\nabla^2 A\big(z^{\star}_c+s\Delta_c\big)\Delta_c\,ds$ with $\nabla^2A(u)=\operatorname{diag}(q)-qq^{\top}$, $q=\operatorname{softmax}(u)$. For any probability vector $q$ and any $v$ with $\sum_k v_k=0$, the variance identity gives $v^{\top}(\operatorname{diag}(q)-qq^{\top})v=\sum_k q_k(v_k-m)^2$ with $m=q^{\top}v$, and
\[
\sum_k q_k(v_k-m)^2\;\ge\;p_{\min}\sum_k(v_k-m)^2\;=\;p_{\min}\big(\|v\|^2+Vm^2\big)\;\ge\;p_{\min}\|v\|^2 .
\]
Hence $\mathrm{KL}(p^{\star}_c\|\hat p_c)\ge\tfrac{p_{\min}}{2}\|\Delta_c\|^2$, and averaging over the $M$ contexts,
$\mathbb{E}_c[\mathcal{L}_c-\mathcal{L}_c^{\star}]\ge\tfrac{p_{\min}}{2M}\|\widetilde{Z}-\widetilde{L}^{\star}\|_F^{2}$.
Since $\operatorname{rank}(\widetilde{Z})\le d+1$ at every parameter setting, the Eckart--Young theorem bounds the Frobenius distance below by the tail of the spectrum: $\|\widetilde{Z}-\widetilde{L}^{\star}\|_F^{2}\ge\sum_{i>d+1}\sigma_i^2(\widetilde{L}^{\star})$. Combining the two displays gives \eqref{eq:gap}. The tail sum is positive exactly when $\operatorname{rank}(\widetilde{L}^{\star})>d+1$.
\end{proof}

{\color{black}Three consequences bear on the experiments, and the first is a delimitation. The floor is \emph{not active at the vocabulary head of the models trained here}: the head width is $d=d_{\mathrm{model}}=1024$ while the byte vocabulary gives centered target rank at most $V-1=255$, so $d+1$ exceeds the maximal possible rank and the vocabulary floor is zero for every model in this study. The proposition therefore does not explain the present byte-level results at the vocabulary, and we state this plainly; it serves as general motivation, and it delimits where readout rank could matter: heads narrower than their vocabulary, as arise at subword scale (\Cref{sec:scaling}), and, by analogy, ranking operations whose width is far below the number of alternatives they rank. Attention is such an operation: each head scores $T=1024$ positions through queries and keys of dimension $D=16$, so an attention-level analogue of the rank constraint is far from vacuous there. Second, where a floor is active it constrains the whole trajectory, not one step: no schedule, optimizer, or backbone takes a rank-limited head below it, so per-step progress must decay as training approaches it, while a scorer of higher effective rank faces no such ceiling. Third, the proposition should not be read as a claim that Born attention removes the softmax bottleneck: both H-models retain the tied softmax vocabulary head, with the same (inactive) vocabulary-rank situation as their baselines. The Born score is better described as a richer \emph{internal interaction rule}, a quadratic score with signed cross terms, whose relation to the floor is the attention-level analogy above rather than an instance of the proposition. H-Mamba uses neither a Born score nor a modified head, so no readout-rank mechanism of any kind is available to it, and its acceleration is most naturally attributed to the recurrence (\Cref{sec:evolution}); we return to this division of labor in \Cref{sec:mechanism}. The full $\Omega(N^{2})$ mechanism concerns the idealized Born vocabulary readout, which the deployable substrate does not instantiate.}

\subsection{Exact transport of relative phase}
\label{sec:phasetransport}
The second statement concerns the recurrence. The quantity that the quadratic readout consumes is relative phase: by the two-state computation of \Cref{sec:readout}, the score depends on the off-diagonal density-matrix coordinates $\rho_{jk}=\psi_j\bar\psi_k$ through their arguments. {\color{black}The following lemma is deliberately narrow: it shows that for an isolated homogeneous component, the recurrence \eqref{eq:rec} accumulates phase additively and transports the phase \emph{coordinate} without distortion, at unit sensitivity, regardless of the gates. What it does not assert, and what \Cref{rem:lemscope} delimits, is that loss gradients through the phase are constant: the observable effect of a phase, and the gradient that returns through it, are scaled by the participating magnitudes.}

\begin{lemma}[Exact transport of relative phase]
\label{lem:phase}
Consider the homogeneous recurrence $\ket{\psi_{t+1}}=a_t\odot\ket{\psi_t}$ with $a_{t,k}=\alpha_{t,k}\lambda_{t,k}$, $\alpha_{t,k}\in(0,1)$, $|\lambda_{t,k}|=1$, and write $\theta_{t,k}=\arg\lambda_{t,k}$. Let $\rho_{jk,t}=\psi_{t,j}\bar\psi_{t,k}$ be an off-diagonal coordinate of the density matrix. Then for every horizon $T$,
\begin{equation}
\arg\rho_{jk,T}=\arg\rho_{jk,0}+\sum_{t<T}\big(\theta_{t,j}-\theta_{t,k}\big),
\qquad
|\rho_{jk,T}|=|\rho_{jk,0}|\prod_{t<T}\alpha_{t,j}\,\alpha_{t,k}.
\label{eq:transport}
\end{equation}
In particular, $\partial\arg\rho_{jk,T}\big/\partial\theta_{s,j}=1$ for every $s<T$, independently of the horizon and of the gates, whereas every derivative of the magnitude carries the accumulated contraction $\prod_t\alpha$.
\end{lemma}

\begin{proof}
Unrolling the recurrence gives $\psi_{T,k}=\psi_{0,k}\prod_{t<T}a_{t,k}$, and therefore $\rho_{jk,T}=\psi_{0,j}\bar\psi_{0,k}\prod_{t<T}a_{t,j}\bar a_{t,k}$. Since $a_{t,j}\bar a_{t,k}=\alpha_{t,j}\alpha_{t,k}\,e^{i(\theta_{t,j}-\theta_{t,k})}$, taking arguments and moduli yields \eqref{eq:transport}; the argument is additive in the phases and free of the gates. Differentiation is immediate. {\color{black}By linearity of the recurrence, the same identities apply to each input contribution $P^{-1}_r\odot b_{s+r}$ of the inhomogeneous solution \eqref{eq:scanform} \emph{separately}; the phase of their sum is a nonlinear function of the parts and can change through interference, so the transport identity governs the parts, not the aggregate.}
\end{proof}

\begin{remark}[Scope of the lemma]
\label{rem:lemscope}
{\color{black}Three delimitations. First, the lemma does not claim a better Jacobian spectrum. The state-to-state Jacobian of \eqref{eq:rec} is $\operatorname{diag}(a_t)$, whose singular values are the gates $\alpha_{t,k}$; a real gated diagonal recurrence at the same gates has the identical spectrum. Any analysis that inspects only the singular values of the recurrence therefore cannot distinguish the two substrates. The distinction resides in the complex structure: the transition acts on each coordinate as a rotation composed with a scaling, the rotation is an isometry of the phase, and the phase is where the substrate stores the information the readout consumes. In a real substrate every information-bearing coordinate is a magnitude, so every information-bearing sensitivity decays with the gates; in the complex substrate the sensitivity of the phase \emph{coordinate} is exactly one at every lag, by \eqref{eq:transport}. Second, coordinate sensitivity is not gradient magnitude. A loss reads $\arg\rho_{jk}$ only through cross terms of size $2|\rho_{jk}|$, so the gradient arriving at a phase parameter is the unit coordinate sensitivity multiplied by the contracted magnitude $\prod_t\alpha_{t,j}\alpha_{t,k}$. The correct statement of the advantage is therefore \emph{decoupling}, not unconditional losslessness: a complex channel can hold its gates near one and transmit phase at full strength while other channels forget, whereas a real channel's single modulus must serve as both its memory and its forgetting, so retention and transmission cannot be assigned to different coordinates. Third, the lemma concerns forward transport of the represented quantity; the backward pass inherits the same rotational structure because the adjoint of a rotation is a rotation, but whether trained models actually hold carrier gates near one is an empirical question. The per-layer gradient measurements of \Cref{sec:gradstab} are consistent with their doing so, and are offered as supporting diagnostics rather than as proof of long-range gradient transmission. A convergence-rate analysis of the joint recurrence-plus-readout system, for instance through the conditioning of the Fisher information, remains open and is noted in \Cref{sec:limitations}.
}\end{remark}

{\color{black}It is worth stating plainly what the two results do and do not contribute, since both have just been delimited. Neither is a derivation of the measured step ratios. \Cref{prop:gap} is inactive at the vocabulary of these experiments and bears on the design only through the analogy to narrow ranking operations; \Cref{lem:phase} concerns an isolated component and yields a decoupling statement rather than a gradient guarantee. What they establish is the \emph{direction} of two candidate mechanisms and, more importantly, the reason the substrate was built as it was: a score whose reachable set is not confined by the width of the state, and a recurrence whose information coordinate is not the coordinate that contraction attenuates. The empirical claims of this paper rest on the measurements of \Cref{sec:results}, not on these statements, and the two are kept separate deliberately. With that scope fixed, the results still make complementary predictions: a score-side advantage wherever the rank of a ranking operation binds, which the attention realization tests, and a transport-side advantage wherever gradients must cross many positions, which the state-space realization tests. \Cref{sec:setup} describes that comparison.}

\section{Experimental Setup}
\label{sec:setup}
The goal of the experiments follows from \Cref{sec:theory}: to measure the per-step advantage of the substrate while excluding capacity, tuning, corpus, and backbone as alternative explanations. We evaluate four models, the real and complex realizations of two backbones, on three byte-level corpora under a single training protocol. {\color{black}The design is intended to make the substrate the principal systematic difference within each pair, and to distinguish substrate effects from effects of dataset scale; the respects in which the paired models nonetheless differ are recorded in \Cref{sec:limitations}.}

\paragraph{Models.}
The four models are the real Mamba and its complex realization H-Mamba, and a real Transformer and its complex realization H-Transformer. {\color{black}The real Transformer uses standard softmax dot-product attention; the complex realization replaces it with the sum-normalized Born score of \eqref{eq:born-attn}, a linear-attention-style kernel, so the attention pair differs in the functional form of the score, including its normalization, which is the readout half of the substrate axis (\Cref{sec:limitations} records the caveat that the pair is parameter-matched but not architecturally identical).} The real Mamba is a selective state-space model with state dimension 16 and a $\Delta$-projection rank of 64; the real Transformer is a pre-normalized decoder with rotary position embeddings and dot-product attention. Both follow their published configurations~\citep{gu2023mamba,su2021roformer}, and neither is weakened: the real Transformer attains the lowest terminal loss of any model in this study on enwik8 ($1.384$ bpc, \Cref{tab:bpc}). All four share the configuration of \Cref{tab:params}, read out through a weight-tied byte embedding, and are matched in parameter count to within $0.02\%$ as described in \Cref{sec:realizations}.

\paragraph{Data.}
We model text at the byte level~\citep{xue2022byt5}: the vocabulary is the 256 byte values, and performance is reported as bits per character (bpc), the cross-entropy in nats divided by $\ln 2$. This is the average number of bits the model needs to encode the next byte, and lower is better. Byte-level modeling removes the tokenizer as a confound. {\color{black}We note for clarity that it also renders the vocabulary-level floor of \Cref{prop:gap} inactive, since the head width exceeds the byte vocabulary; the rank considerations of \Cref{sec:theory} pertain, if anywhere, to the attention scores, whose per-head width $D=16$ is far below the $T=1024$ positions they rank.} The three corpora span $150\times$ in size: enwik8~\citep{enwik8}, 100~MB of Wikipedia markup; OpenWebText~\citep{openwebtext}, a 3~GB stream of web text; and FineWeb~\citep{fineweb}, a 15~GB subset of CommonCrawl. Each corpus is read as raw bytes, partitioned $90/5/5$ into training, validation, and test, and segmented into non-overlapping windows of 1024 bytes, the target at each position being the next byte. The test portions are held out and untouched; every number in this paper is computed on the validation portions, and the released logs permit test evaluation.

\paragraph{Optimization.}
A single protocol is applied without modification to all four models, so no element of the schedule can favor one substrate. The optimizer is AdamW~\citep{kingma2015adam,loshchilov2019decoupled} with $\beta_1=0.9$, $\beta_2=0.95$, and decoupled weight decay $0.1$ on tensors of rank at least 2; the learning rate follows a cosine schedule~\citep{loshchilov2017sgdr} from $6\times 10^{-4}$ to $6\times 10^{-5}$ with 750 warmup steps; gradients are clipped to unit norm~\citep{pascanu2013difficulty}; computation is \texttt{bf16} with per-layer activation checkpointing. These values follow the published recipe for language models of this size~\citep{brown2020language}. None was tuned to either substrate, and we deliberately performed no per-model search: unequal tuning effort would confound the substrate comparison, whereas a common published recipe favors neither side by construction. The residual possibility that the shared recipe suits one substrate better than the other is discussed in \Cref{sec:limitations}. Training is data-parallel over four A100-80GB accelerators at 256 sequences each, an effective batch of 1024 sequences, or $1{,}048{,}576$ tokens, per step. Validation loss is evaluated every 20 steps and the best checkpoint retained. \Cref{tab:hparams} lists the protocol in full.

\begin{table}[t]
\centering
\caption{Complete optimization protocol, applied without modification to all twelve runs.}
\label{tab:hparams}
\small
\begin{tabular}{@{}ll@{}}
\toprule
Hyperparameter & Value \\
\midrule
Optimizer & AdamW (fused) \\
Peak / minimum learning rate & $6\times10^{-4}$ / $6\times10^{-5}$ \\
Schedule & cosine, $750$ warmup steps \\
Weight decay & $0.1$ (tensors of rank $\ge 2$) \\
$\beta_1,\beta_2$ & $0.9,\,0.95$ \\
Gradient clipping & unit norm \\
Precision & \texttt{bf16} autocast \\
Activation checkpointing & per layer \\
Batch (per device / effective) & $256$ / $1024$ sequences \\
Tokens per step & $1{,}048{,}576$ \\
Steps & $1000$ (OWT/FineWeb), $1500$ (enwik8) \\
Validation interval & every $20$ steps \\
Hardware & $4\times$ A100-80GB \\
Framework & PyTorch~\citep{paszke2019pytorch} \\
\bottomrule
\end{tabular}
\end{table}

\paragraph{Budget.}
The optimization budget is 1000 steps on OpenWebText and FineWeb and 1500 on enwik8; unless stated otherwise, models are compared at the common budget of 1000 steps. This is the appropriate basis for a claim about per-step optimization, and the choice is consequential on the smallest corpus. Beyond roughly 1000 steps the complex models, which extract more from each step, begin to overfit the 100~MB of enwik8; their validation loss climbs toward $2.2$--$2.3$ bpc by step 1500 while the slower real Mamba is still improving. \Cref{sec:smalldata} treats this behavior, which is a consequence of the substrate's sample efficiency rather than a defect of it.

\paragraph{Runs, controls, and reproducibility.}
Three properties make the comparison a test of the substrate rather than of tuning: the protocol is identical across models; the counterparts are parameter-matched to within $0.02\%$; and every training log, together with the parser that converts logs into the tables and figures below, is released, so each quantity is re-derivable, not asserted. Twelve runs cover the four models on the three corpora. The real-Mamba run on FineWeb stopped immediately after its step-1000 validation evaluation, which is the common comparison point, so it enters every comparison below; we note the single place where its extended trajectory is missing. The methodological caveat is that each cell is a single run; \Cref{sec:limitations} states which conclusions this qualifies and calibrates, against the smallest effects in the paper, which effects seed variance could plausibly explain.
\section{Results}
\label{sec:results}
{\color{black}Four models, three corpora, and one protocol yield twelve runs. The per-step comparison comes first (\Cref{sec:perstep}), then the behavior after the learning-rate warmup ends (\Cref{sec:postwarmup}), terminal quality and the behavior at the data limit (\Cref{sec:finalbpc}), and finally throughput and the gradient measurements (\Cref{sec:throughput}).}

\subsection{Per-step convergence}
\label{sec:perstep}
\label{sec:absolute}
\Cref{fig:val_bpc} plots validation loss against optimization step for all runs. We summarize each pair by the ratio of steps needed to reach a common loss. For a target $\ell$, the ratio is the step at which the real model first attains $\ell$, divided by the step at which the complex model first attains it; each step is obtained by linear interpolation between the validation evaluations (every 20 steps) that bracket the target. \Cref{tab:learning_speed} reports this ratio at three targets; \Cref{tab:absolute} reports the underlying step counts, and \Cref{fig:steps} draws them. The targets $2.5$, $2.2$, and $2.0$ bpc lie in the interval over which all four models are still descending, so the ratio measures the rate of convergence, not the eventual separation of terminal losses. Lower targets are excluded because the slower models do not attain them within the budget, which would turn a measurement into a lower bound. As a representative instance: on enwik8, H-Mamba first reaches $2.0$ bpc at step 191; the real Mamba reaches it at step 575.

\begin{table}[t]
\centering
\caption{\textbf{Optimization steps to reach a target validation loss, real divided by complex.} A value of $3.0$ indicates that the real model requires three times as many steps as its complex counterpart to reach the stated bits per character. Ratios are read from the curves of \Cref{fig:val_bpc}; a larger value is a larger per-step advantage for the substrate.}
\label{tab:learning_speed}
\small
\setlength{\tabcolsep}{6pt}
\begin{tabular}{@{}llccc@{}}
\toprule
\textbf{Pair} & \textbf{Corpus} & \textbf{$2.5$} & \textbf{$2.2$} & \textbf{$2.0$} \\
\midrule
H-Mamba / Mamba & enwik8  & $3.2$ & $3.0$ & $3.0$ \\
H-Mamba / Mamba & OpenWebText & $3.1$ & $3.0$ & $2.9$ \\
H-Mamba / Mamba & FineWeb & $3.0$ & $2.8$ & $3.5$ \\
\midrule
H-Transformer / Transformer & enwik8  & $2.2$ & $2.0$ & $1.8$ \\
H-Transformer / Transformer & OpenWebText & $2.5$ & $2.3$ & $2.2$ \\
H-Transformer / Transformer & FineWeb & $2.5$ & $2.3$ & $2.1$ \\
\bottomrule
\end{tabular}
\end{table}

\begin{table}[t]
\centering
\caption{Step at which each model first attains a target validation loss (linear interpolation between evaluations). The ratios of \Cref{tab:learning_speed} are the real counts divided by the complex counts.}
\label{tab:absolute}
\small
\setlength{\tabcolsep}{5pt}
\begin{tabular}{@{}llcccc@{}}
\toprule
& & \multicolumn{2}{c}{State-space} & \multicolumn{2}{c}{Attention} \\
\cmidrule(lr){3-4}\cmidrule(lr){5-6}
Corpus & Target & H-Mamba & Mamba & H-Trans. & Trans. \\
\midrule
enwik8 & $2.5$ & 78 & 252 & 109 & 237 \\
enwik8 & $2.2$ & 121 & 360 & 156 & 317 \\
enwik8 & $2.0$ & 191 & 575 & 224 & 405 \\
\midrule
OpenWebText & $2.5$ & 75 & 237 & 101 & 252 \\
OpenWebText & $2.2$ & 118 & 353 & 151 & 350 \\
OpenWebText & $2.0$ & 197 & 581 & 219 & 484 \\
\midrule
FineWeb & $2.5$ & 79 & 240 & 105 & 264 \\
FineWeb & $2.2$ & 133 & 373 & 158 & 366 \\
FineWeb & $2.0$ & 217 & 758 & 247 & 514 \\
\bottomrule
\end{tabular}
\end{table}

\begin{figure}[t]
\centering
\includegraphics[width=\linewidth]{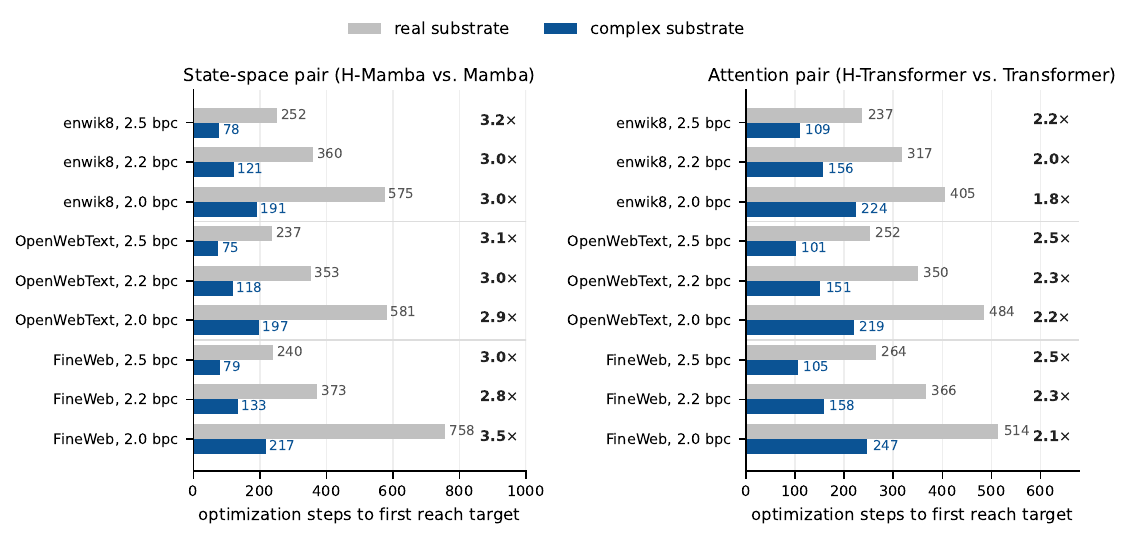}
\caption{\textbf{Steps to first reach each validation target}, drawn from the counts of \Cref{tab:absolute}. The complex model (blue) requires fewer steps than its real counterpart (gray) in all 18 comparisons; annotated ratios are those of \Cref{tab:learning_speed}. The complex state-space model passes $2.5$ bpc within 75--79 steps on every corpus, a level the real model reaches after 237--252 steps.}
\label{fig:steps}
\end{figure}

Three features of \Cref{tab:learning_speed} determine the interpretation of the result.

\paragraph{Stability across thresholds.}
Across the three columns of the table, the ratio is nearly constant as the target descends from $2.5$ to $2.0$ bpc. A transient advantage of initialization would appear as a ratio decaying toward one with descending target; instead, the separation is maintained through the steep portion of the learning curve, so the complex models do not merely start faster and get overtaken as training proceeds.

\paragraph{Stability across corpora.}
Within each backbone, the rows of the table agree across corpora that differ by $150\times$ in size: approximately 3 for the state-space pair on enwik8, OpenWebText, and FineWeb alike, and approximately 2 for the attention pair. An effect tied to the statistics of a particular corpus would vary across these rows, and it does not.

\paragraph{Stability across backbones.}
The four models constitute a $2\times 2$ design over backbone and substrate, and both complex cells improve on their real cells by a stable, backbone-characteristic factor. The most informative cell is H-Mamba. {\color{black}Its vocabulary readout is an ordinary tied softmax, so no readout mechanism is available to it, and its acceleration, the larger of the two, is most naturally attributed to the recurrence, where \Cref{sec:evolution} locates the retained benefit of the substrate.}

\paragraph{Crossings.}
Within the common budget, the complex curve lies below its real counterpart at every recorded evaluation for five of the six pairs. The exception is the attention pair on enwik8, where the rotary Transformer, which attains the lowest terminal loss of any model in this study on this corpus, crosses below H-Transformer within the budget as the complex model begins to saturate the 100~MB corpus; it finishes at $1.384$ against $1.516$ bpc (\Cref{tab:bpc}). Beyond the budget, and again on enwik8 only, the real Mamba eventually overtakes the by-then overfitting H-Mamba (\Cref{sec:smalldata}). {\color{black}Both crossings occur in the regime where the sample efficiency of the substrate exhausts the smallest corpus; on the web corpora no curve crosses at any point. The measured targets are, however, all crossed during the learning-rate warmup, and \Cref{sec:postwarmup} examines what happens once it ends.}

\begin{figure}[t]
\centering
\includegraphics[width=\linewidth]{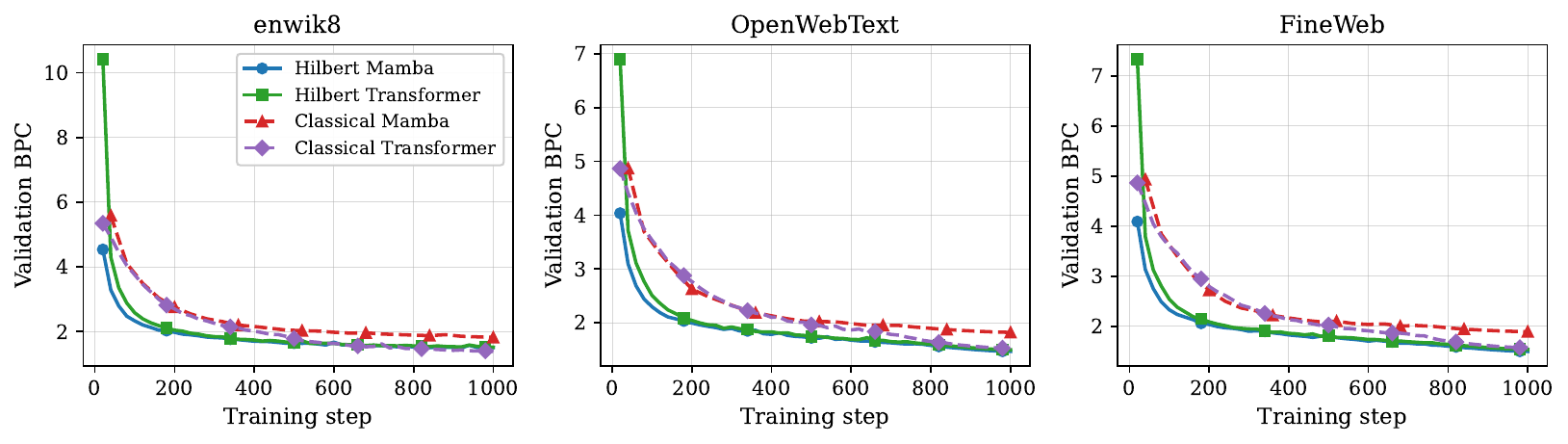}
\caption{\textbf{Per-step convergence on the three corpora.} Validation bpc against optimization step. Within the measured range of targets, each complex realization (solid) reaches every loss value in fewer steps than its real counterpart (dashed); the horizontal gap at a fixed loss is the ratio tabulated in \Cref{tab:learning_speed}. On enwik8 the complex curves turn upward beyond roughly step 1000 as the 100~MB corpus is overfit (\Cref{sec:smalldata}).}
\label{fig:val_bpc}
\end{figure}

{\color{black}\subsection{Behavior after the warmup phase}
\label{sec:postwarmup}
Every target in \Cref{tab:learning_speed} is crossed during the learning-rate warmup, which occupies the first 750 of the 1000 steps of the budget (\Cref{tab:hparams}); the complex models cross all measured targets between steps 75 and 247, and the real models cross the $2.5$-bpc target in the same interval. A comparison confined to that interval invites a specific alternative explanation: that the ramp of the schedule, rather than the substrate, produces the separation. The schedule is identical across all four models, so this would have to be an interaction between the ramp and the substrate rather than a difference in treatment; but the interaction is not excluded by the consistency of the ratios, since the same schedule is used in all twelve runs and an interaction would be equally consistent.

The released logs settle the question directly, because they record validation loss on both sides of the warmup boundary. \Cref{fig:gaptraj} plots the validation gap, real minus complex, against optimization step, and \Cref{tab:postwarmup} reports it at the boundary and at the end of the budget. An advantage manufactured by the ramp should contract once the ramp stops.

\begin{table}[t]
\centering
\caption{\textbf{Validation gap (real $-$ complex, bpc) at the end of warmup and at the end of the budget.} A positive value means the complex model leads. The state-space gap widens after the warmup ends on both web corpora; the attention gap contracts on all three. On enwik8 the complex models are at the data limit by this point (\Cref{sec:smalldata}), which confounds the comparison there.}
\label{tab:postwarmup}
\small
\begin{tabular}{@{}lcccccc@{}}
\toprule
& \multicolumn{3}{c}{State-space pair} & \multicolumn{3}{c}{Attention pair} \\
\cmidrule(lr){2-4}\cmidrule(lr){5-7}
Corpus & step 750 & step 1000 & change & step 750 & step 1000 & change \\
\midrule
enwik8      & $+0.374$ & $+0.315$ & $-0.059$ & $-0.049$ & $-0.132$ & $-0.083$ \\
OpenWebText & $+0.321$ & $\mathbf{+0.354}$ & $\mathbf{+0.034}$ & $+0.083$ & $+0.026$ & $-0.058$ \\
FineWeb     & $+0.368$ & $\mathbf{+0.396}$ & $\mathbf{+0.028}$ & $+0.114$ & $+0.034$ & $-0.081$ \\
\bottomrule
\end{tabular}
\end{table}

\begin{figure}[t]
\centering
\includegraphics[width=\linewidth]{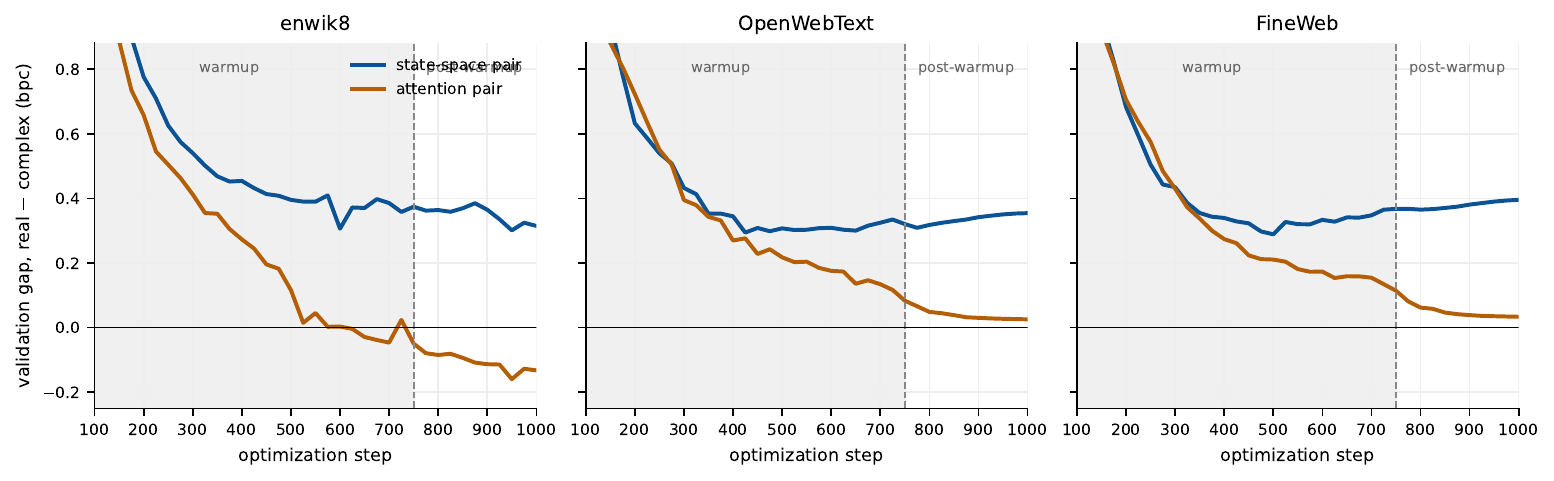}
\caption{\textbf{Validation gap against optimization step, on both sides of the warmup boundary.} Shaded region: the 750-step learning-rate warmup; dashed line: its end. Positive values indicate the complex model leads. The state-space gap (blue) is flat to increasing after the warmup ends on the two web corpora; the attention gap (orange) decays monotonically throughout on all three corpora. On enwik8 the complex models reach the data limit within the budget, so its late behavior reflects overfitting (\Cref{sec:smalldata}) rather than the schedule.}
\label{fig:gaptraj}
\end{figure}

The two pairs behave differently, and we report the difference as part of the result. For the state-space pair, the gap increases after the warmup ends on both web corpora, from $+0.321$ to $+0.354$ bpc on OpenWebText and from $+0.368$ to $+0.396$ on FineWeb. This is the opposite of what an artifact of the ramp would produce, and it is observed on precisely the corpora where the data limit is not reached within the budget; the one corpus where the gap contracts, enwik8, is the one where the complex models have exhausted a 100~MB training split (\Cref{sec:smalldata}). {\color{black}We therefore regard the warmup explanation as inconsistent with the state-space measurements, within the budget examined. Two limits of this evidence should be stated with it. The interval after the boundary is short, 250 steps of the 1000-step budget, though the trend is monotone across the twelve validation evaluations it contains; and it is not a constant learning rate but the cosine decay phase, so what the measurement excludes is an artifact of the \emph{ramp}, not every possible interaction between the schedule and the substrate. A run at constant learning rate remains the control that would exclude the wider class, and we identify it as such in \Cref{sec:limitations}.}

For the attention pair the picture is different: the gap contracts on all three corpora, on both sides of the boundary, reaching $+0.026$ and $+0.034$ bpc at step 1000 on the two web corpora, and extrapolation of the trend would place it near zero shortly beyond the budget. The attention advantage measured in \Cref{tab:learning_speed} is therefore concentrated in early training, and we do not claim that it persists. This asymmetry is consistent with the mechanism decomposition of \Cref{sec:mechanism}: the conditioning mechanism, which the state-space pair exercises, addresses a structural attenuation that is present at every step, whereas the score mechanism operates against a real attention baseline that is already well conditioned and that closes the gap as training proceeds.

}\subsection{Terminal performance}
\label{sec:finalbpc}
\label{sec:scaledep}
\label{sec:smalldata}
Faster convergence is informative only if it is not obtained at the expense of the converged loss. \Cref{tab:bpc} reports validation loss at the common 1000-step budget. H-Mamba improves on the real Mamba on all three corpora, by $0.315$ to $0.396$ bpc, equivalently a $20$ to $24\%$ reduction in per-byte perplexity. H-Transformer improves on the real Transformer on both web corpora, by $0.026$ and $0.034$ bpc, and trails it on enwik8 for the reason given above. The per-step advantage of \Cref{tab:learning_speed} is the principal finding; {\color{black}the present table indicates that, within the measured budget and in the regime where data is not the binding constraint, {\color{black}the acceleration is not obtained at the expense of the converged loss. These step-1000 values are measured 250 steps after the warmup ends, and \Cref{tab:postwarmup} shows that the state-space margins are still growing at that point while the attention margins are contracting.}

\begin{table}[t]
\centering
\caption{\textbf{Validation bpc at the 1000-step common budget} (lower is better). Daggers mark the complex realizations; the FineWeb real-Mamba entry is its step-1000 evaluation, the last its log records. On enwik8 the complex realizations overfit beyond the budget (\Cref{sec:smalldata}).}
\label{tab:bpc}
\small
\begin{tabular}{@{}lccc@{}}
\toprule
Model & enwik8 & OpenWebText & FineWeb \\
\midrule
Mamba & 1.819 & 1.820 & 1.892 \\
H-Mamba$^\dagger$ & \textbf{1.504} & \textbf{1.466} & \textbf{1.496} \\
\midrule
Transformer & \textbf{1.384} & 1.522 & 1.568 \\
H-Transformer$^\dagger$ & 1.516 & \textbf{1.496} & \textbf{1.534} \\
\bottomrule
\end{tabular}
\end{table}

\paragraph{Dependence on corpus size.}
The terminal margin of H-Mamba over the real Mamba increases monotonically with corpus size: $0.315$ bpc on enwik8, $0.354$ on OpenWebText, and $0.396$ on FineWeb. The attention margin follows the same ordering, from negative on enwik8 to $0.026$ on OpenWebText and $0.034$ on FineWeb. Were the substrate's benefit a fixed representational offset, one would expect the margins to contract on larger corpora as the models approach the achievable loss; the observed widening is in the opposite direction. {\color{black}A mechanism for the widening is not established by this paper: the vocabulary-level floor of \Cref{prop:gap} is inactive here (\Cref{sec:floor}), so it cannot supply one, and {\color{black}the reading we consider most natural, that more diverse data exercises more of the phase capacity the recurrence transports, is a hypothesis. The ordering is nonetheless informative for a different question: the step-1000 margin is smallest on the corpus where the complex models have exhausted the data and largest on the largest corpus, which is the ordering expected if the contraction of the enwik8 margin reflects the data limit and not the schedule (\Cref{sec:postwarmup}).} Because each margin is a single run and the increments lie within what a second seed could perturb, the monotone widening itself should also be read as directional rather than established.}

\paragraph{Behavior at the data limit.}
On enwik8 the complex models lead at every point of the common budget, reach their minimum near step 1000, and subsequently overfit; the still-descending real Mamba overtakes them late in the extended budget. This behavior is not in tension with the central result but is a corollary of it. A model that extracts more from each step also reaches, in fewer steps, the point at which a finite corpus ceases to constrain it, and 100~MB affords a 253M-parameter model few enough distinct contexts that this point falls within the budget. The web corpora, larger by factors of 30 and 150, contain no such point within the budget, and the complex models improve monotonically throughout. Two practical consequences follow. At small data scales, the sample efficiency of the substrate translates into an earlier need for regularization or early stopping, the standard treatments; it does not call for a modification of the substrate itself. For evaluation methodology, a fixed large budget on a small corpus penalizes the fastest learner for exhausting the data first, so comparisons between substrates should be made at matched loss or matched data in addition to matched steps. The training curves (\Cref{app:trainloss}) coincide with the validation curves up to the point of overfitting, {\color{black}which is consistent with the advantage on the web corpora reflecting learning rather than memorization.}

\subsection{Throughput and gradient propagation}
\label{sec:throughput}
\label{sec:gradstab}
Optimization steps are one of the two factors in wall-clock training time, the other being the cost of a step, and we report measured throughput while distinguishing the fair comparisons from the unfair. The fused Born-attention kernel of H-Transformer processes approximately $80{,}000$ tokens per second, against approximately $160{,}000$ for the real Transformer. Both employ fused attention kernels (cf.~\citealp{dao2022flashattention}), so the comparison is fair, and the factor of two in per-step cost offsets the factor of two in per-step convergence: at present, the two attention models reach a given loss in comparable wall-clock time, and the net value of the substrate for that pair lies in the data consumed rather than in the time elapsed. The diagonal scan of H-Mamba processes approximately $48{,}000$ tokens per second; the reference implementation of the real Mamba, which evaluates its scan without kernel fusion, processes approximately $7{,}000$. We do not report that ratio as a speedup, since it reflects the absence of an optimized baseline kernel rather than a property of the substrate. Combining step counts with measured throughput yields elapsed-time estimates: on OpenWebText at the $2.0$ bpc target, approximately $0.8$ hours of training for H-Transformer (219 steps at $80{,}000$ tokens per second) against approximately $0.9$ hours for the real Transformer (484 steps at $160{,}000$), consistent with the parity stated above; the corresponding state-space estimate awaits an optimized baseline kernel. The per-step result of \Cref{sec:perstep} is independent of these considerations, as it counts optimization steps under a common protocol and is unaffected by kernel throughput. A per-step arithmetic accounting is given in \Cref{app:flops}.

{\color{black}The gradient measurement anticipated in \Cref{sec:evolution}, and consistent with \Cref{lem:phase}, provides supporting evidence for the conditioning mechanism. (Global-norm clipping rescales all layers by a common factor and therefore does not affect the per-layer \emph{spread} analyzed here.)} \Cref{fig:grad_norm} reports per-layer gradient norms over training for H-Mamba and the real Mamba. Across the 24 layers of the real Mamba, the per-layer gradient norm spreads by more than an order of magnitude, the early layers receiving one to two orders of magnitude less signal than the late layers, which is the signature of geometric attenuation over the 1024-step sequence. Across the same 24 layers of H-Mamba the spread is a factor of approximately two. {\color{black}This pattern is what the decoupling account predicts when trained carrier channels hold their gates near one (\Cref{rem:lemscope}); it is a supporting diagnostic for the conditioning mechanism, not direct proof of long-range gradient transmission.} {\color{black}Gradient conditioning and the score's interaction rule are nonetheless distinct mechanisms that act jointly in the attention realization,} and their clean separation requires the factorial ablation specified in \Cref{sec:mechanism}; {\color{black}the present measurement is consistent with the retained mechanism being operative, and does not indicate that it acts alone.}

\begin{figure}[t]
\centering
\includegraphics[width=0.82\linewidth]{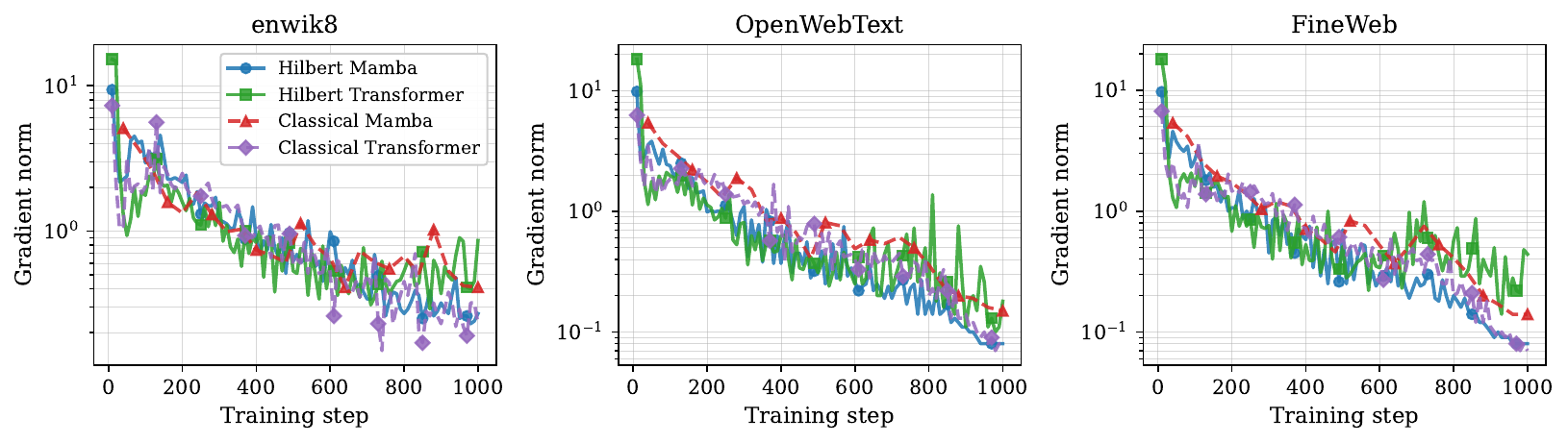}
\caption{\textbf{Per-layer gradient norm over training} for H-Mamba and the real Mamba. H-Mamba maintains all 24 layers within a factor of approximately two of one another, whereas the real Mamba exhibits a spread of more than an order of magnitude between its early and late layers, consistent with the geometric attenuation of a contractive real recurrence over the 1024-step sequence.}
\label{fig:grad_norm}
\end{figure}
\section{Discussion}
\label{sec:discussion}
{\color{black}The experiments show that a single substitution, the deployable complex substrate in place of the real one, is accompanied by a reduction in the number of steps to every measured target, by a factor of approximately three for the state-space backbone and two for the attention backbone, across corpora spanning a $150\times$ range of size.} We examine the locus of the effect (\Cref{sec:universality}), the decomposition of the mechanisms that plausibly produce it (\Cref{sec:mechanism}), its consequences for training cost (\Cref{sec:economics}), its expected behavior under the scaling variables held fixed here (\Cref{sec:scaling}), and its relation to the idealized substrate (\Cref{sec:pathforward}).

\subsection{The locus of the effect}
\label{sec:universality}
The two backbones are dissimilar in precisely the respect that matters for an optimization claim. The state-space backbone maintains a fixed-dimensional summary of the past and updates it multiplicatively, so a gradient reaching an early position must traverse the entire recurrence. The attention backbone recomputes a weighted combination of all past positions at every step, so an early position is reached directly, in one hop. Any explanation located in the routing, whether in how gradients thread the recurrence or in what attention can express, applies to one backbone and not the other. An explanation that survives the swap must live in what the two models share, and after the parameter matching of \Cref{sec:realizations} they share exactly one non-standard component: the substrate.

The substrate contributes two properties that are indifferent to routing. {\color{black}The first is phase transport: the argument of a state coordinate moves through time without distortion, whether the path is a 1024-step recurrence or an attention sum, because the Cayley phase has unit modulus by construction; the strength at which it is observed remains scaled by the accumulated magnitudes (\Cref{rem:lemscope}).} The measurements of \Cref{sec:gradstab} show this property operating in the backbone where its absence is most damaging. {\color{black}The second is interference: the scoring operation can combine two present components with opposite signs, at a strength set by their magnitudes, which enlarges the function class of the score without a dedicated gating module (the cancellation mechanism of \Cref{sec:readout}).} Neither property refers to the manner in which positions are connected. {\color{black}This is the sense in which substrate and backbone are close to orthogonal axes of design, and the stable, backbone-characteristic ratios of three and two are what one would expect when a common cause meets two different amounts of headroom. The independence claim requires one qualification, which \Cref{sec:postwarmup} supplies: the early-phase acceleration appears in both backbones, whereas an advantage that persists past the warmup and continues to widen is observed, within this budget, only for the state-space backbone. What transfers across backbones is therefore the early effect; the persistence of the effect is established here only where the conditioning mechanism operates.}

Two experiments would test the attribution directly. If the effect is substrate-borne, a third backbone given the same complex state and Born score, for instance a long-convolution or gated linear-attention model, should accelerate by its own stable factor, whereas a backbone-specific account predicts no acceleration there. If phase transport is the operative property for recurrent backbones, the state-space ratio should grow with sequence length, since the attenuation $\rho^{T}$ of the real baseline worsens geometrically in $T$ while the phase channel does not. Neither experiment requires new theory, and either could falsify the substrate attribution.

\subsection{Decomposition of the mechanism}
\label{sec:mechanism}
\label{sec:ablation}
Two mechanisms could produce the acceleration, and our experiments separate them only partially. {\color{black}The first is the interaction rule of the score: the Born score is a quadratic interaction with signed cross terms, and, by the attention-level analogy of \Cref{sec:floor} (the vocabulary-level floor itself is inactive here), a richer score can escape rank-type ceilings that a $D=16$-dimensional real dot product faces when ranking $1024$ positions.} The second is optimization conditioning: the recurrence transports gradients through the phase (\Cref{lem:phase}, \Cref{sec:gradstab}) and supplies interference to the state dynamics. The pattern of results assigns the mechanisms at two of the four corners of the design. {\color{black}The vocabulary readout of H-Mamba is a tied softmax identical to its baseline's, and it uses no Born score, so no readout mechanism is available to it; its threefold acceleration is accordingly attributed to conditioning. The Born score of H-Transformer is a genuinely richer interaction rule, so the score mechanism is available to it in addition to conditioning.}

Surprisingly, the backbone in which both mechanisms are available gains less: a factor of two rather than three. The resolution lies in the headroom available to each baseline. The real Transformer is already well conditioned; its gradient paths are one hop long, and no attenuation of the form $\rho^{T}\approx 3.4\times 10^{-5}$ afflicts it, so the conditioning mechanism, the larger of the two, has little left to repair. The real Mamba, whose contractive recurrence attenuates propagated signal by factors of order $10^{-5}$ over its 1024-step sequences (\Cref{sec:evolution}), gives the conditioning mechanism the most room to operate. {\color{black}This account yields the prediction noted in \Cref{sec:universality}: the state-space ratio should grow with sequence length while the attention ratio should not, since only the former baseline deteriorates with $T$. It also anticipates the post-warmup trajectories of \Cref{sec:postwarmup}. A conditioning defect of the real recurrence is present at every step of training, so an advantage that repairs it should not expire when the schedule stops ramping, which is what the widening state-space gap shows; a score-side advantage over a baseline that is already well conditioned has no comparable structural deficit to keep exploiting, and the contracting attention gap is what that predicts. We note this as a consistency between the mechanism decomposition and an independent measurement, not as a confirmation of the decomposition, which the factorial cells below would test.}

What remains unresolved is the interaction of the two mechanisms, and the experiment that resolves it is a factorial design over two binary factors: the field of the recurrence, real or complex, and the score, affine--softmax or Born. Two of the four cells are the models of this paper. The two missing cells are a complex recurrence with a purely softmax pipeline, which isolates conditioning, and a real recurrence with a Born score, which isolates rank; comparing each missing cell with the real baseline measures one mechanism alone, and the amount by which the full substrate exceeds the sum of the two measures their interaction. A complementary control requires no complex arithmetic at all: a mixture-of-softmaxes head~\citep{yang2018breaking} raises the rank of a real readout, and if it closed the gap on the state-space backbone the readout would be implicated after all, although the fact that H-Mamba attains its acceleration with an ordinary softmax head already argues otherwise. We specify these cells in detail because they are the most informative experiments the study leaves undone; the available compute did not permit them, and \Cref{sec:limitations} records this as the principal open question of the paper.

\subsection{Steps, tokens, and wall-clock time}
\label{sec:economics}
At a fixed batch size, a reduction in optimization steps is an equal reduction in the tokens consumed. Each step processes $1{,}048{,}576$ tokens (\Cref{tab:hparams}), so reaching $2.0$ bpc on FineWeb costs H-Mamba approximately $0.23$B tokens against approximately $0.79$B for the real Mamba (\Cref{tab:absolute}), a saving of roughly $0.57$B tokens for that single target. This is the most direct value of the substrate, and it accrues where training is data-bound rather than compute-bound: specialized domains, private corpora, low-resource languages, and byte-level settings in which the corpus, and not the accelerator budget, is the scarce resource. The overfitting of \Cref{sec:smalldata} is the same property observed at the opposite extreme: on a 100~MB corpus, reaching a given loss on one third of the data means exhausting the data three times sooner.

Wall-clock time is governed by the product of the step count and the step cost, and the two pairs stand differently in this respect. For the attention pair, the measured throughputs ($80{,}000$ against $160{,}000$ tokens per second) exactly offset the step ratio of two, so the substrate at present purchases data efficiency at wall-clock parity. For the state-space pair the comparison requires an optimized real baseline kernel, which does not exist in our setup; a comparison against the unfused reference scan would overstate the substrate's contribution, and we do not make it. The break-even condition is nonetheless fixed by the step ratio alone. With a step ratio of three, the complex scan at $48{,}000$ tokens per second reaches any fixed target sooner in wall-clock time than any real kernel running below approximately $144{,}000$ tokens per second at this configuration. Whether an optimized real scan exceeds that figure is an engineering question, not a scientific one; the accounting of \Cref{app:flops} indicates that the complex kernel, at present bound by memory bandwidth, not arithmetic, has unexploited headroom of its own. Until such kernels exist, the established value of the substrate is in steps and tokens, and its value in hours remains conditional.

\subsection{Behavior under scaling}
\label{sec:scaling}
Three quantities held fixed in this study govern the extrapolation of the result: the vocabulary size, the state width, and the model scale. The vocabulary is the 256 byte values, which happens to equal $N^{2}=16^{2}$. The coincidence would matter for the idealized Born vocabulary readout, whose rank advantage binds when $N^{2}\ge V$; it is immaterial to the models actually trained, whose vocabulary head is a tied softmax that scales as any language-model head does. Because the measured advantage originates in the recurrence and the attention score, it carries no substrate-specific dependence on $V$, and we expect it to persist at subword vocabularies. This expectation is untested, and verifying it is the first experiment we would undertake. The component that scales unfavorably is the idealized readout: at vocabularies of 30k to 128k tokens it demands a state width of order $\sqrt{V}$ and a readout cost of $O(NV)$, or else the surrender of the exact rank advantage to a softmax.

The state width $N=16$ is narrow. The expressivity of the substrate and the cost of the scan grow linearly in $N$, while the rank ceiling of the idealized readout grows as $N^{2}$; whether the per-step advantage grows, saturates, or shrinks with $N$ is not determined here and bears directly on the value of the substrate at scale. The model scale is fixed at 253M parameters and the budget at 1000 to 1500 steps. The regime in which the idealized readout becomes decisive, a vocabulary far exceeding $N^{2}$, as with a 50k subword vocabulary against a few hundred state dimensions, is precisely the regime this study cannot reach. The scaling of the effect in $N$, $V$, and parameter count, and the point at which restoring the idealized readout becomes worthwhile, are the principal quantities left undetermined; settling them requires experiments at the several-hundred-million to billion-parameter scale, for which the established scaling methodology applies directly~\citep{kaplan2020scaling,hoffmann2022training}.

\subsection{Relation to the idealized substrate}
\label{sec:pathforward}
\label{sec:designaxis}
\label{sec:designspace}
The substrate evaluated here is deliberately the conservative member of a family, and the family has three axes. The first is the generator of the evolution, ranging from the diagonal phase used here, through block-diagonal and banded Hermitian generators that couple groups of coordinates, to the dense input-dependent Hamiltonian of the idealized model; the diagonal choice obtains the $O(N)$ scan at the price of inter-coordinate coupling within the recurrence, which the readout must then supply. The second is the norm behavior of the recurrence, ranging from the contractive gate used here to exact unitarity; the gate obtains selective forgetting at the price of the exact gradient guarantee. The third is the placement of the Born rule, ranging from the attention score used here to the vocabulary readout; the attention placement obtains a cost independent of vocabulary size at the price of the $\Omega(N^{2})$ readout advantage. On all three axes we occupy the vertex nearest a production system, chosen so that the substrate could be hosted by an existing backbone and trained without incident.

The relaxations define the remaining experiments, which are the specific ones the theory of \citet{nebli2026quantum} designates rather than generic future work. Restoring exact unitarity, together with a forgetting mechanism compatible with it, such as unitary dynamics on an extended state with explicit sink coordinates, would test whether norm preservation adds anything beyond the phase transport that the relaxed recurrence already provides; the measurements of \Cref{sec:gradstab} suggest the residual is small, and a null result would itself be informative, licensing the cheaper contractive form. Moving the Born rule to the vocabulary readout at subword scale would test whether the $\Omega(N^{2})$ rank advantage materializes as a training advantage where the rank ceiling actually binds. Since the conservative vertex already yields a stable factor of two to three in optimization steps, a systematic traversal of this space, and the identification of the vertex that is optimal at a given scale, appears at least as well posed a research program as a further iteration on the backbone, whose axis is by now densely explored while the substrate axis remains nearly untouched.

\section{Limitations}
\label{sec:limitations}
Several limitations bound the conclusions, and we state each together with the experiment that would remove it. Every cell of the experimental design is a single training run, so no error bars are available, and we calibrate what seed variance could plausibly explain against the smallest effects we report. The attention margins on the web corpora, $0.026$ to $0.034$ bpc, and the monotone widening of the margin with corpus size are effects we ourselves treat as potentially seed-level; they should be read as directional until replicated. The state-space margins, $0.315$ to $0.396$ bpc, are an order of magnitude larger than that self-declared noise scale, and the corresponding step ratios of approximately three recur across nine independent corpus-target cells (\Cref{tab:learning_speed}); we therefore regard their direction as established. Three runs per cell on a single corpus would settle both questions. {\color{black}{\color{black}Four further measurement caveats apply. Most target crossings occur during or near the 750-step learning-rate warmup. The complex models cross every measured target inside it, between steps 75 and 247, and the real models cross the 2.5-bpc target there as well. The step ratios of \Cref{tab:learning_speed} are therefore measured in a regime where the schedule is nonstationary, though identical across models. \Cref{sec:postwarmup} addresses this for the state-space pair, whose advantage widens after the warmup ends on both web corpora, which an artifact of the ramp would not do; it does not address it for the attention pair, whose advantage contracts throughout and which we accordingly describe as an early-training effect of untested persistence. A run at constant learning rate, with no warmup and no decay, would settle the question for both pairs directly, and is the control we would add first.} Crossing steps are obtained by linear interpolation between validation evaluations spaced 20 steps apart, which bounds their resolution. The paired models are parameter-matched but not functionally identical: the real Mamba carries a parameter-matching SwiGLU branch, and the attention pair differs in the functional form of the score, including its normalization. And the per-layer gradient plots of \Cref{sec:gradstab} are supporting diagnostics for the conditioning mechanism, not direct measurements of long-range gradient transport.} All models, moreover, share one size, 253M parameters, and one budget regime, 1000 to 1500 steps, which is far from convergence on the web corpora. The result is therefore a statement about the measured range of the loss, and the persistence of the per-step advantage under budgets and models one to two orders of magnitude larger is the most consequential question the study leaves open; extending a single pair on one web corpus by an order of magnitude in steps would answer it directly.

Two further limitations concern interpretation, not measurement. The deployable substrate relaxes both properties on which the representational theory of \citet{nebli2026quantum} rests, so the $\Omega(N^{2})$ separation describes the idealized model and not the models trained here, and \Cref{prop:gap} bounds the readout in isolation, under an assumption on the probabilities traversed during training, without accounting for the magnitude of the full effect. The two mechanisms that plausibly produce the acceleration, {\color{black}the conditioning of the optimization by the recurrence and the richer interaction rule of the Born score, are likewise confounded in the attention model, which possesses both;} the factorial cells that would separate them, a complex recurrence with a softmax score, a real recurrence with a Born score, and the mixture-of-softmaxes control, were not run for want of compute (\Cref{sec:mechanism}). Finally, the throughput comparison for the state-space backbone is uninformative, because the real baseline runs an unfused reference scan, and all wall-clock figures are specific to A100-80GB hardware in half precision; an optimized real-Mamba kernel at matched configuration would allow the break-even condition of \Cref{sec:economics} to be tested rather than stated.

A third group of limitations concerns cost, engineering, and evaluation scope. Complex arithmetic carries a constant-factor overhead, and current accelerators provide no native complex tensor-core path, so every complex operation decomposes into real ones; the fused Born-attention kernel consequently runs at half the throughput of its real counterpart, and the step advantage of the attention pair does not yet translate into wall-clock savings. The implementation depends on two nonstandard components for numerical stability, the clamped log-domain scan of \Cref{sec:scan} and the single-precision accumulation of the fused kernel, without which \texttt{bf16} training of the substrate is unstable in our experience; adopting the substrate therefore carries a real, if modest, engineering cost. The training protocol, while identical across models, was not re-tuned per model, and a per-model learning-rate sweep could move the reported ratios in either direction; such a sweep is inexpensive relative to the factorial ablation and is worth performing alongside it. The theoretical account is partial in a specific sense: {\color{black}\Cref{prop:gap} is inactive at the vocabulary scale used here and \Cref{lem:phase} concerns an isolated component, so together they indicate the direction of each mechanism rather than accounting for the measured magnitudes; a convergence-rate or conditioning analysis} of the joint recurrence-plus-readout system, for instance through the Fisher information, is not attempted. Lastly, the evaluation is confined to language-modeling loss; calibration, robustness, out-of-distribution behavior, and analyses of the learned representations, including the distribution of learned phases and the singular-value spectra of the states, are natural companion studies that our results do not address.

\section{Conclusion}
\label{sec:conclusion}
\citet{nebli2026quantum} established, for an idealized model class on synthetic tasks, that a complex substrate is representationally more powerful than a real one. This paper establishes the corresponding statement about training. {\color{black}A deployable form of the substrate, contractive where the idealized model is unitary and with the Born rule confined to the attention score, reaches every measured validation loss in approximately one third of the optimization steps of the real Mamba and one half of those of the real Transformer, from the first hundred steps of training, across corpora spanning a $150\times$ range of size, at matched parameter counts. The two backbones differ in what happens next: the state-space advantage continues to widen after the learning-rate warmup ends, while the attention advantage contracts toward zero, so the early acceleration is common to both backbones and its persistence is established here only for the recurrence.} Stated compactly: replacing the number system of the state and the form of the score, while leaving the routing mechanism untouched, reduced the optimization steps of two structurally unrelated models by factors of two to three, and {\color{black}the indifference of the effect to the routing supports attributing it to the substrate.} The relaxations that made the substrate deployable define the experiments that remain: restoring exact unitarity, and returning the Born rule to the vocabulary readout at the scale where its rank advantage is predicted to bind. On the present evidence, the field over which a model computes is a design decision of the same standing as the mechanism by which it attends.

\bibliographystyle{unsrtnat}
\bibliography{refs}

\appendix
\section{Substrate Pseudocode}
\label{app:pseudocode}
\Cref{alg:cayley,alg:scan} implement the recurrence \eqref{eq:rec} of \Cref{sec:evolution} without sequential evaluation, and \Cref{alg:born} implements the attention score \eqref{eq:born-attn} of \Cref{sec:readout}.

\begin{algorithm}[h]
\caption{Diagonal Cayley transition coefficient}
\label{alg:cayley}
\begin{algorithmic}[1]
\REQUIRE phase parameters $\boldsymbol{\phi}\in\mathbb{R}^{N}$, gate logits $\mathbf{u}\in\mathbb{R}^{N}$
\STATE $\boldsymbol{\lambda}\gets(1+i\boldsymbol{\phi})\oslash(1-i\boldsymbol{\phi})$ \COMMENT{$|\lambda_k|=1$, $\arg\lambda_k=2\arctan\phi_k$}
\STATE $\boldsymbol{\alpha}\gets\sigma(\mathbf{u})$ \COMMENT{selective gate, $\alpha_k\in(0,1)$}
\STATE \textbf{return} $\mathbf{a}\gets\boldsymbol{\alpha}\odot\boldsymbol{\lambda}$ \COMMENT{$|a_k|=\alpha_k$}
\end{algorithmic}
\end{algorithm}

\begin{algorithm}[h]
\caption{Chunked parallel scan for $\psi_{t+1}=a_t\odot\psi_t+b_t$}
\label{alg:scan}
\begin{algorithmic}[1]
\REQUIRE initial state $x_0\in\mathbb{C}^{N}$, and $\log|\mathbf{a}|$, $\arg\mathbf{a}$, $\mathbf{b}$ over a chunk of length $C$
\STATE $\ell\gets\operatorname{clamp}(\operatorname{cumsum}(\log|\mathbf{a}|),-30,0)$;\quad $\vartheta\gets\operatorname{cumsum}(\arg\mathbf{a})$
\STATE $P\gets\operatorname{polar}(e^{\ell},\vartheta)$;\quad $P^{-1}\gets\operatorname{polar}(e^{-\ell},-\vartheta)$
\STATE \textbf{return} $P\odot\big(x_0+\operatorname{cumsum}(\mathbf{b}\odot P^{-1})\big)$ \COMMENT{final state carried to next chunk}
\end{algorithmic}
\end{algorithm}

\begin{algorithm}[h]
\caption{Causal Born attention, per head}
\label{alg:born}
\begin{algorithmic}[1]
\REQUIRE complex queries, keys, values $q,k,v$; temperatures $\tau$; head width $D$
\STATE $s_{ij}\gets|\braket{q_i}{k_j}|^{2}/(D\,\tau_i)$ for $j\le i$ \COMMENT{squared modulus, not softmax}
\STATE $\psi_i\gets\big(\textstyle\sum_{j\le i}s_{ij}v_j\big)\big/\big(\sum_{j\le i}s_{ij}\big)$ \COMMENT{causal running-sum normalization}
\STATE \textbf{return} $\psi$ \COMMENT{fused kernel, single-precision accumulation}
\end{algorithmic}
\end{algorithm}

\section{Per-Step Arithmetic Accounting}
\label{app:flops}
The diagonal Cayley recurrence needs $O(N)$ arithmetic per step per head: forming the transition coefficient and updating the state each cost about $6N$ real operations, roughly $12N$ in total, against the $2N^{2}$ floating-point operations of a dense real transition, a ratio of approximately 3 at $N=16$. This ratio is not realized in wall-clock time, for three reasons. At $N=16$ the recurrence kernel is bound by memory bandwidth, not arithmetic: the state resides in registers, and the cost is moving the projections through high-bandwidth memory. Complex arithmetic carries a constant-factor overhead in our implementation. And the projection costs are shared with the real model, so they cancel in the comparison. We therefore report measured throughput in \Cref{sec:throughput} and claim no kernel speedup for the state-space backbone, whose real baseline runs an unfused reference scan.

\section{Training Loss}
\label{app:trainloss}
\Cref{fig:train_loss} reports training loss against optimization step, the counterpart of the validation curves of \Cref{fig:val_bpc}. The two agree: the complex realizations lead on the training objective by the same per-step margins they show on validation, and the separation appears at the same early step. {\color{black}This argues against the reading that the validation acceleration is an artifact of the split, since a model that reduced training loss no faster than its counterpart would be unlikely to lead on validation by the observed margin.} Training and validation diverge only on enwik8 beyond the budget, at the overfitting point identified in \Cref{sec:smalldata}. The real Mamba is omitted from this figure because its training loop logs a sequence-summed rather than token-mean loss, which puts its training curve on an incomparable scale; its validation figures, computed by the shared evaluation path, are unaffected and appear throughout the main text.

\begin{figure}[h]
\centering
\includegraphics[width=\linewidth]{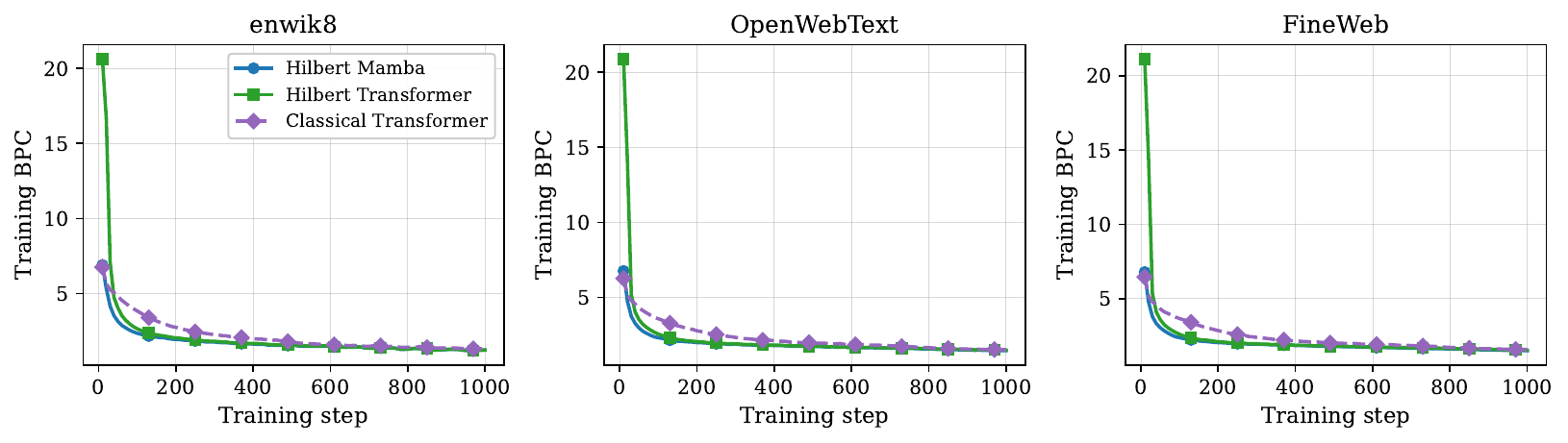}
\caption{\textbf{Training loss against optimization step} for the complex realizations and the real Transformer. The per-step ordering reproduces that of the validation curves in \Cref{fig:val_bpc}, establishing that the acceleration is a property of learning, not of the validation split.}
\label{fig:train_loss}
\end{figure}

\section{Reproducibility}
\label{app:repro}
Each corpus is read as raw bytes, partitioned $90/5/5$ into training, validation, and test, and segmented into non-overlapping windows of 1024 bytes, the target being the input advanced by one byte. Bits per character is the cross-entropy in nats divided by $\ln 2$. The training loop records the loss every 10 steps and the validation loss every 20, over 20 validation batches during training and over the full validation split at completion. All twelve training logs are released, together with the parser that produces every table and figure from them; every quantity reported in \Cref{tab:learning_speed,tab:absolute,tab:bpc} is re-derivable from the logs. The FineWeb real-Mamba log ends at its step-1000 evaluation, which is the value \Cref{tab:bpc} reports. For reference, the components a reimplementation requires are specified as follows: the state initialization and frequency grid in \Cref{sec:init}, the Cayley parameterization and gate in \Cref{sec:evolution}, the scan in \Cref{sec:scan} and \Cref{app:pseudocode}, the block structure and normalization in \Cref{sec:realizations}, the parameter matching in \Cref{sec:realizations}, and the optimizer, schedule, and precision in \Cref{tab:hparams}. The training code and the fused kernels accompany the release.

\end{document}